\documentclass{ecai} 

\usepackage{latexsym}
\usepackage{graphicx}
\usepackage{color}
\usepackage{times}
\usepackage{soul}
\usepackage{url}
\usepackage[hidelinks]{hyperref}
\usepackage[utf8]{inputenc}
\usepackage[small]{caption}
\usepackage{graphicx}
\usepackage{amsmath}
\usepackage{amsthm}
\usepackage{booktabs}
\usepackage{algorithm}
\usepackage{algorithmic}
\usepackage[switch]{lineno}
\usepackage{xspace}
\newtheorem{theorem}{Theorem}

\usepackage[T1]{fontenc}    
\usepackage{url}            
\usepackage{booktabs}       
\usepackage{amsfonts}       
\usepackage{nicefrac}       
\usepackage{microtype}      
\usepackage{xcolor}         
\usepackage{amsthm}
\usepackage{amsmath}
\usepackage{amssymb}
\usepackage{mathtools}
\usepackage{paralist}
\usepackage{diagbox} 
\usepackage{array}

\newtheorem{lemma}{Lemma}
\newtheorem{definition}{Definition}

\newtheorem{proposition}{Proposition}

\newcommand{\mysubsection}[1]{\medskip\noindent\textbf{#1}}

\usepackage[inline]{enumitem}

\newcommand{\relu}{ReLU}

\newcommand{\yes}{\textit{Yes}}
\newcommand{\no}{\textit{No}}

\newcommand{\PComplexity}{PTIME}
\newcommand{\NPComplexity}{NP}
\newcommand{\NPCompleteComplexity}{NP-complete}
\newcommand{\coNPComplexity}{co-NP}

\newcommand{\StoPComplexity}{$\Sigma_{2}^{P}$} 
\newcommand{\StoPCompleteComplexity}{$\Sigma_{2}^{P}$-complete} 
\newcommand{\countPComplexity}{\#P}
\newcommand{\countPCompleteComplexity}{\#P-complete}

\newcommand{\z}{\textbf{z}}
\newcommand{\x}{\textbf{x}}
\newcommand{\X}{\{1,\dots,n\}}

\newcolumntype{P}[1]{>{\centering\arraybackslash}p{#1}}

\begin{document}

\begin{frontmatter}
	
	
	\paperid{123} 
	
	
\title{Hard to Explain: On the Computational Hardness\\ of In-Distribution 
Model Interpretation}

\author[A]{\fnms{Guy}~\snm{Amir}\thanks{Corresponding
 Author. Email: guy.amir2@mail.huji.ac.il.}\footnote{Equal contribution.}}
\author[A]{\fnms{Shahaf}~\snm{Bassan}\footnotemark}
\author[A]{\fnms{Guy}~\snm{Katz}} 

\address[A]{The Hebrew University of Jerusalem, Jerusalem, Israel}

%

	
	\begin{abstract}
            The ability to interpret Machine Learning (ML) models is becoming 
            increasingly essential. However, despite significant progress in 
            the field, there remains a lack of rigorous characterization 
            regarding the innate interpretability of different models.
			In an attempt to bridge this gap, recent work has demonstrated that 
			it is possible to \emph{formally} assess interpretability by 
			studying the computational complexity of explaining the decisions 
			of various models. In this setting, if explanations for a 
			particular model can 
			be
                obtained efficiently, the model is considered interpretable 
                (since it can be explained ``easily''). However, if generating 
                explanations over an ML model is computationally intractable, 
                it is considered uninterpretable. 
                Prior research identified two key 
		factors that influence the complexity of interpreting an ML model: 	
		\begin{inparaenum}[(i)]
			\item the \emph{type} of the model (e.g., neural networks, decision 
			trees, etc.); and
			\item the \emph{form} of explanation (e.g., contrastive 
			explanations, Shapley values, etc.).
		\end{inparaenum}
  In this work, we claim that a third, important factor must also be 
  considered for this analysis 
  --- the 
		\emph{underlying distribution} over which the explanation is obtained. 
                Considering the underlying distribution is key in 
                avoiding explanations that 
		are \emph{socially misaligned}, i.e., convey information that is biased 
		and unhelpful to users. We demonstrate the significant
                influence of the underlying distribution on the resulting overall 
		interpretation complexity, in two settings:
                \begin{inparaenum}[(i)]
                  \item prediction models paired with an \emph{external} 
                  out-of-distribution 
                    (OOD) detector; and
                  \item prediction models designed to \emph{inherently} 
                  generate 
                    socially aligned explanations.
                  \end{inparaenum}
                  Our findings prove that the expressiveness of 
		the distribution can significantly influence the overall complexity of interpretation, and identify  essential 
                  prerequisites 
		that a model must possess to generate socially
                aligned explanations. We regard this work as a
                step towards a rigorous characterization of the complexity
                of generating explanations for ML models, and towards gaining a 
                mathematical understanding of their interpretability.
		
	\end{abstract}
	
\end{frontmatter}%

	\setcounter{secnumdepth}{1}
	

	\section{Introduction}

Ensuring the interpretability of ML models is becoming increasingly vital, as 
it enhances their trustworthiness, particularly when deployed in 
safety-critical systems~\cite{huang2020survey}. However, despite significant 
advancements in the field, there remains a notable lack of mathematical rigor 
in understanding the inherent interpretability of various ML models. For 
instance, many fundamental claims within interpretability, such as ``decision 
trees are more interpretable than neural networks", are often regarded as 
folklore and lack sufficient mathematical rigor.

To bridge this gap, work by Barcelo et al.~\cite{BaMoPeSu20} proposes 
assessing the interpretability of an ML model by examining the computational 
complexity involved in generating various types of explanations for it. The idea is that 
if explanations can be efficiently obtained for an ML model, it can be 
considered interpretable. Conversely, if obtaining explanations is 
computationally intractable, the model is deemed uninterpretable. For example, 
while obtaining certain explanation forms for decision trees can be computed in 
polynomial or even linear time, these same tasks become NP-hard for neural 
networks~\cite{BaMoPeSu20, ignatiev2019abduction, huang2021efficiently}. This provides rigorous mathematical evidence that neural 
networks are indeed less interpretable than decision trees in these contexts.

The computational complexity of obtaining explanations was studied 
in a variety of different settings~\cite{BaMoPeSu20, waldchen2021computational, 
BaAmKa24}, in which the computational complexity is 
typically analyzed along two main axes:
\begin{inparaenum}[(i)] 
	\item the model \emph{type} and 
	\item the explanation \emph{form}. 
\end{inparaenum} 
For example, computing Shapley value 
explanations for decision trees can be obtained in polynomial 
time~\cite{arenas2021tractability, van2022tractability}, while obtaining 
minimum size contrastive explanations for neural networks is 
NP-complete~\cite{BaMoPeSu20}.

\mysubsection{The Distribution Component.} In many explainability methods, 
understanding the rationale behind a specific input prediction often involves 
defining an explanation that satisfies certain properties in inputs similar to 
the one being interpreted. For instance, inputs that are identical to the 
original one in most features, with differences in only a few. This approach 
can be problematic because these new inputs might be out-of-distribution (OOD), 
and may deviate substantially from inputs of interest. Hence, the OOD 
inputs may affect the explanation in unexpected ways, and convey unintuitive 
information to users. Hase et al.~\cite{HaXiBa21} refer to 
explanations that 
disregard the input distribution as \emph{socially misaligned}, i.e., convey 
information that is biased and unhelpful to users.

This general OOD phenomenon in explanations is termed ``the OOD problem of 
explainability"~\cite{HaXiBa21} and is encountered in numerous explanation 
forms, including counterfactual explanations~\cite{poyiadzi2020face}, 
contrastive 
explanations~\cite{yu2022eliminating, GoRu22}, sufficient 
explanations~\cite{HaXiBa21, yu2022eliminating, GoRu22}, and Shapley 
values~\cite{sundararajan2020many}. Therefore, many 
practical explanation techniques aim to mitigate the impact of OOD instances, 
making this a crucial aspect of computing more precise 
explanations~\cite{kim2020interpretation, chang2018explaining, 
yi2020information, HaXiBa21, taufiq2023manifold, zintgraf2017visualizing, 
sundararajan2017axiomatic}.

In this work, we argue that evaluating the computational complexity of 
explaining the decision of a model, should not rely solely on the model type 
and the
explanation form, but also on the \emph{underlying distribution} over which 
the explanation is computed. 
The distribution component is crucial for ensuring that the computed 
explanations are socially aligned and meaningful. In this paper, we
illustrate the impact of this factor on the overall interpretation 
complexity, in various settings and scenarios.

\mysubsection{A Running Example.} Consider the task of classifying low-dimensional 
images as either ``0" or ``1". Due to the simplicity of this task, let us 
assume 
that it can be effectively learned using a simple decision tree classifier. 
Given an image classified as ``0", we can interpret the prediction of the 
decision tree using a local, post-hoc explainability method. For instance, we 
can obtain a \emph{sufficient reason} $S$~\cite{ignatiev2019abduction, 
darwiche2020reasons, BaKa23}: a subset of features (in this case, pixels) that, 
when fixed, ensure the image remains classified as ``0'', regardless of the 
assignment of the additional features $\overline{S}$. Fortunately, since this 
task was learned by a decision tree classifier, obtaining a locally minimal 
sufficient reason can be achieved in polynomial 
time~\cite{huang2021efficiently}.

	
	
	However, despite their appeal, sufficient reasons, similarly to other 
	explanation forms, suffer from the OOD problem of 
	explainability~\cite{HaXiBa21, yu2022eliminating, GoRu22}. In this 
	particular case, the sufficient reason $S$ may take into account 
	\emph{OOD assignments} over $\overline{S}$. In other words, setting the 
	pixels of
	$\overline{S}$ to partial images that are OOD
	(e.g., images featuring unrelated digits, or cats) might
	result in the image being classified as ``1''. This will preclude $S$
	from being a sufficient reason --- even if it is one when taking into 
	account only the context of 
	interest (i.e., all in-distribution images of 
	the digits ``0'' or ``1'').

A common solution for bridging this gap is to train another model
to detect OOD inputs, and then use it to dismantle the effect of any
misleading 
assignment~\cite{kim2020interpretation,chang2018explaining,yi2020information,HaXiBa21,taufiq2023manifold,zintgraf2017visualizing}.
However, the task of OOD detection is considered very
challenging, both in theory~\cite{FaLiLuDoHaLi22,PeYi22} and in
practice~\cite{HsShJiKi20,SeAlGoSlNuLu19,BeXiMaZhLiXuZh20} --- as
modeling the feature distribution is often harder than the original
prediction task~\cite{sundararajan2020many}. Hence, obtaining an OOD
classifier may require training a very expressive model, such as a generative 
model that approximates the domain distribution $p_{\phi}(\x)$. For our 
running example, for instance, learning to distinguish between in-distribution
images (``0'' or ``1'') and OOD images (any other possible image) may be a 
substantially harder task than learning to classify images of ``0'' and ``1''. 
Such a task may require the use of a much more expressive model, such as a deep 
generative neural network. 
The complexity of obtaining a
sufficient reason $S$ that ignores the effect of any OOD
assignment may thus be much greater than that of simply explaining the decision 
tree classifier, without considering the distribution. Revisiting our running 
example, the findings in this study demonstrate that performing this task is 
indeed NP-hard, despite the fact that computing such an explanation without 
distribution alignment can be done in polynomial time.

\mysubsection{Paper Structure}. 
In Sec.~\ref{sec:preliminaries}, we start by covering the relevant background 
for this work. Next, in Sec.~\ref{sec:our-general-framework}, 
we examine a wide variety of explanation forms, such
as sufficiency-based, contrastive-based, and counting-based
explanations, and study how they can be formalized to maintain
social alignment. Specifically, we delve into the common scenario where the 
classification model is coupled with an additional component --- an OOD 
detector. 
This detector plays a crucial role in mitigating the impact of OOD 
counterfactuals in explanations, and can be used to align various explanation 
forms with a distribution of interest.
We proceed to demonstrate that diverse 
explanation forms can be
unified through a single framework, which captures their shared
structure. Given an OOD detector, this
framework can be used to preserve the alignment of each of these explanations; 
as well as to study the computational complexity of
obtaining them. 

In Sec.~\ref{sec:social-alignment-complexity} we prove that for any explanation 
matching our abstract form, the complexity of 
interpreting a model is dominated by the complexity of 
interpreting an OOD detector for the same type of explanation.
Since OOD detection is computationally hard~\cite{FaLiLuDoHaLi22}, the task of 
obtaining an \emph{aligned} interpretation of the model may be substantially more 
complex than the misaligned form.

In Sec.~\ref{sec:self-alignment} we study the specific
case of \emph{self-aligned} explanations. Here, our focus
shifts from relying on an external OOD-detection model to the
possibility of utilizing a single model that derives
aligned explanations. Specifically, we focus on the case of efficiently 
producing a single model that serves both as a classifier and as an OOD 
detector, 
given that each of these is realized separately by the same model class. As we 
prove, this capability correlates to the
degree of expressiveness inherent in various ML models 
--- while
some model types possess the required level of expressiveness, others do not.  
We prove these insights for specific model types and show that, assuming 
P$\neq$NP, both neural networks and decision trees have the capability to 
derive self-aligned explanations, while linear classifiers do not. 

Furthermore, related work is covered in Sec.~\ref{sec:related-work}. We 
conclude in 
Sec.~\ref{sec:conclusion}, and discuss the limitations of our 
theoretical framework, as well as potential future work in 
Sec.~\ref{sec:limitations}.

Due to space limitations, we provide only concise overviews of the
proofs of our various claims, and refer the reader to the
appendix for the comprehensive and more detailed proofs.

\section{Preliminaries}
\label{sec:preliminaries}

\subsection{Domain}



We assume a set of $n$ features $\x=(x_1,\ldots,x_n)$,
where the domain of each feature is $x_i\in \{0,1\}$. The entire feature space 
is denoted as $\mathbb{F}=\{0,1\}^n$.
We seek to \emph{locally} interpret the prediction of a binary classifier
$f:\mathbb{F}\to \{0,1\}$, i.e., given an input $\x\in 
\mathbb{F}$, to
explain the prediction $f(\x)$ of the classifier over this specific input. We 
follow common practice in the field, and concentrate on Boolean 
input and output values, to make the presentation 
clearer~\cite{ArBaBaPeSu21,waldchen2021computational,BaMoPeSu20}. 
However, many of our findings are also applicable to scenarios involving 
real-valued data.


\subsection{Complexity Classes and Second-Order Logic (SOL)}
The paper assumes basic familiarity with the common complexity classes of 
polynomial time (PTIME) and nondeterministic polynomial time (\NPComplexity, 
\coNPComplexity).
The
second order of polynomial hierarchy, i.e., \StoPComplexity, which is
briefly mentioned in the paper, is the set of problems that become members of
\NPComplexity{} given an oracle that solves  \coNPComplexity{}
problems in $O(1)$. We also discuss the class \countPComplexity, which 
corresponds to the total number of accepting paths of a polynomial-time 
nondeterministic Turing machine. It is widely believed that \PComplexity 
$\subsetneq$ \NPComplexity $\subsetneq$ \StoPComplexity $\subsetneq$ 
\countPComplexity~\cite{arora2009computational}. 
We use the common convention $L_{1}\leq_{p}L_{2}$ to denote a polynomial-time 
reduction from language $L_{1}$ to $L_{2}$, and $L_{1}=_{p}L_{2}$  to indicate 
that such a reduction exists in both directions.

The paper also makes use of \emph{second-order logic} (SOL) formulas --- a 
generalization of the first-order predicate logic. In both logic forms, 
existential or universal
quantifiers are applied to each variable or subset thereof, so that the formula
evaluates to either true or false. 
However, we chose SOL formulas for our abstraction due to
their high expressivity (in contrast to first-order logic queries
suggested in~\cite{ArBaBaPeSu21}), as they can also encode an 
explanation size, which is infeasible with FOL.
As such, SOL-based queries are rigorous enough to 
enable the
formulation of general proofs that hold for any explanation within
this framework.  
For each SOL formula $Q$, we  define $\#Q$ 
as the corresponding counting problem over that formula --- which counts the 
\emph{number} of satisfying assignments for $Q$. Given a finite number of 
inputs, implying a finite logic-based model, each
SOL formula is associated with a specific complexity class within the
polynomial hierarchy, and with a corresponding counting class.

\subsection{Explainability Queries}

We follow prior work~\cite{BaMoPeSu20, BaAmKa24} and define an
\emph{explainability query}, denoted $Q$, which represents some form of 
interpretation. $Q$ takes both $f$ and $\x$ as inputs, and it outputs 
information regarding the
interpretation of $f(\x)$. In line with previous
work~\cite{marques2020explaining,waldchen2021computational,ArBaBaPeSu21,arenas2022computing,
 BaAmKa24},
our emphasis is on explainability queries that output an answer to a 
decision problem --- providing a definite yes/no answer or, in the 
case of $Q$ being a counting problem, a numerical value. For example, $Q$ can 
provide a yes/no answer to the question \emph{is a specific subset of features 
a sufficient reason?} It can also \emph{count} the number of possible 
assignments in which the prediction is altered, or maintained.

\subsection{Models}
The techniques presented in this work are applicable to a diverse set
of model classes. Still, we focus our attention on a few
popular models, located at the extremities of the interpretability spectrum: 
decision trees, linear classifiers, and neural networks. Specifically, we 
address 
Free Binary Decision Diagrams (FBDDs), which serve as an extension of decision 
trees, along with Perceptrons and Multi-Layer Perceptrons (MLPs) employing ReLU 
activations. An exact formalization of these models appears in the appendix.

\section{Socially Aligned Explainability Queries}
\label{sec:our-general-framework}

\subsection{Context Indicator}

To cope with the undesired effects of OOD input assignments, 
we consider some \emph{context} $\mathbf{C}\subseteq \mathbb{F}$ over
which an explanation is to be provided. Intuitively, context
$\mathbf{C}$ denotes the entire potential set of in-distribution 
inputs that we take into consideration when providing an explanation,
while disregarding the effect of any OOD assignment from $\mathbb{F}\setminus 
\mathbf{C}$.
Because describing  the context $\mathbf{C}$ explicitly is clearly non-trivial, 
in our framework, we instead assume the existence of a \emph{context indicator} 
$\pi:\mathbb{F}\to \{0,1\}$: a binary classifier over a specific context 
$\mathbf{C}$, i.e., $\pi(\x)= \mathbf{1}_{\{\x\in \mathbf{C}\}}$. 

Naturally, assuming the existence of a context
indicator $\pi$ that perfectly captures the desired context \textbf{C} is 
non-trivial as well. For instance, in our running example, this requires $\pi$ 
to identify any possible image of either ``0'' or ``1''. Nevertheless, 
practical tools were shown to be able to approximate such domains, for example, 
by using generative-model-based OOD classifiers, trained to learn the data 
distribution $p_{\phi}(\x)$~\cite{CuCiNaXiLiXiNi22,LiHuLaRu21}. In these 
particular scenarios, the indicated $\textbf{C}$ can be seen as a mere 
approximation of the true, intended context. 




\subsection{Socially Aligning Explainability Queries}
\label{queries}

Model interpretability is subjective, and this has led to the design
of multiple forms of explanations in recent years. We focus here on
a few widely used explanation forms, and analyze them rigorously.


\mysubsection{Sufficiency-Based
  Explanations.} A common
definition of an explanation for a model $f$'s decision with respect to an 
input 
$\x$ is that of a \emph{sufficient reason}~\cite{ignatiev2019abduction, 
darwiche2020reasons, BaKa23}. A sufficient reason is a subset of features 
$S\subseteq \X$ such that, when fixed to the corresponding values in $\x$, 
determine that the prediction remains $f(\x)$, regardless of the other 
features' assignments~\cite{BaMoPeSu20,marques2020explaining}. 
This notation is used quite often, and aligns with commonly used 
explainability techniques~\cite{ribeiro2018anchors}.
Formally put:
\begin{equation}
	\label{explanation_definition}
	\forall(\z\in \mathbb{F}).\quad [f(\x_{S};\z_{\Bar{S}})= f(\x)]
\end{equation}
where $(\x_{S};\z_{\Bar{S}})$ denotes an assignment in which the values of $S$ 
are taken from $\x$ and, the remaining values (i.e., from $\overline{S}$), are taken 
from $\z$. 

		
		

Given a context $\mathbf{C}$, indicated by  $\pi$, a socially aligned 
sufficient reason is defined as follows~\cite{yu2022eliminating,GoRu22}:
\begin{equation}
	\label{explanation_distribution_csr_definition_2}
	\forall(\z\in \mathbb{F}).\quad [\pi(\x_{S};\z_{\Bar{S}})=1\to 
	f(\x_{S};\z_{\Bar{S}})=f(\x)]
\end{equation}



		
		



A widely observed convention in the literature is that smaller
sufficient reasons (relative to the size of $|S|$) are more meaningful than 
larger
ones~\cite{ignatiev2019abduction,BaMoPeSu20,halpern2005causes}. Consequently,
it is interesting to consider  \emph{cardinally minimal sufficient reasons}. 
Clearly, these can also be obtained with respect to $\pi$. This leads us to our 
first explainability query:


\vspace{0.5em} 

\noindent\fbox{%
	\parbox{\columnwidth}{%
		\mysubsection{MSR (Minimum Sufficient Reason)}:
		
		\textbf{Input}: Model $f$, input $\x$, context indicator $\pi$, and 
		integer $k$.
		
		\textbf{Output}: \yes{}, if there exists a sufficient reason $S$ for 
		$f(\x)$ with respect to $\pi$ such that $|S| \leq k$, and \no{} 
		otherwise.
	}%
}

\vspace{0.5em} 


We note that we can consider the case of socially misaligned queries as a 
trivial case of this definition, in which the context indicator is the constant 
function  
$\pi\coloneqq\mathbf{1}$,
indicating the entire input space as in-distribution.


\medskip
\noindent
\textbf{Contrastive/Counterfactual-Based 
Explanations.} A different
approach to interpreting a model is by observing subsets of features that,
when altered, may cause the classification of the model to
change~\cite{ignatiev2019abduction,BaMoPeSu20}. These are referred to as \emph{contrastive
	explanations} or \emph{contrastive reasons}, and the corresponding values 
	are referred to as
\emph{counterfactual explanations}. We define a subset $S\subseteq \X$
as contrastive if altering its values may cause the original classification 
$f(\x)$ to change:
\begin{equation}
	\exists \z\in \mathbb{F}.\quad [f(\x_{\Bar{S}};\z_{S})\neq f(\x)]
\end{equation}

To avoid counterfactual OOD assignments, a contrastive subset $S$ can be 
obtained with 
respect to a context indicator $\pi$~\cite{yu2022eliminating}, by encoding:
\begin{equation}
	\exists \z\in \mathbb{F}.\quad [\pi(\x_{\Bar{S}};\z_{S})=1\wedge 
	f(\x_{\Bar{S}};\z_{S})\neq f(\x)]
\end{equation}

Similarly to sufficient reasons, smaller contrastive reasons tend to
be more meaningful. Here, too, it is usually more informative to focus
on cardinally minimal contrastive reasons, as expressed in the following 
explainability query: 

\vspace{0.5em} 

\noindent\fbox{%
	\parbox{\columnwidth}{%
		\mysubsection{MCR (Minimum Change Required)}:
		
		\textbf{Input}: Model $f$, input $\x$, context indicator $\pi$, and 
		integer $k$.
		
		\textbf{Output}: \yes{}, if there exists some contrastive reason $S$ 
		such that $|S| \leq k$ for $f(\x)$ with respect to $\pi$, and \no{} 
		otherwise.
	}%
}

\vspace{0.5em} 

\mysubsection{Counting-Based 
Explanations.}
 Finally, another common explainability 
form is based on exploring the number of assignment completions for maintaining 
(or altering) a specific classification~\cite{izza2021efficient,darwiche2002knowledge,waldchen2021computational}. 
As with previous explanation forms, we redefine the problem
to avoid counting OOD completions, which may cause the social misalignment of 
the corresponding interpretation. In order to do so, we define the completion 
count $c$ of $S$ with respect to $\pi$ as:
\begin{equation}
	c(S):= |\{\z\in \{0,1\}^{|\overline{S}|}, \pi(\x_{\Bar{S}};\z_{S})=1, 
	f(\x_{\Bar{S}};\z_{S}) \neq f(\x)\}|
\end{equation}

\vspace{0.5em} 

\noindent\fbox{%
	\parbox{\columnwidth}{%
		\mysubsection{CC (Count Completions)}:
		
		\textbf{Input}: Model $f$, input $x$, context indicator $\pi$, and 
		subset of features $S$.
		
		\textbf{Output}: The completion count $c(S)$ of $f(x)$ with respect to 
		$\pi$.
	}%
}

\vspace{0.5em} 

The widely used Shapley values~\cite{sundararajan2020many}, which serve as a 
common form of explanation~\cite{lundberg2017unified, 
sundararajan2017axiomatic}, can also be characterized as a type of counting 
problem~\cite{van2022tractability, arenas2021tractability}.

 

\subsection{Abstract Query Form}
\label{general_query_form_section}

Many of the explanation forms studied in the literature, including
the aforementioned ones, become more meaningful when the effect of  
OOD counterfactuals are reduced. For analyzing how distributions affect the complexity of obtaining explanations not only for one specific explanation, but for a wide array of explanation forms,
we proceed to define \emph{abstract} explainability queries. We then provide 
general results regarding the computational complexity 
of obtaining this abstract form of explanation.


The task of obtaining each of the explanation types discussed so far
can be achieved by invoking a 
decision procedure for determining whether or not
$f(\x_{\Bar{S}};\z_{S}) = f(\x)$ (or for solving the corresponding counting 
problem).
These decision procedures receive a partial assignment $(\x_{\Bar{S}};\z_{S})$
of a given input $\x$, which fixes some features of $\x$ while allowing
 the rest to change according to an arbitrary $\z$; and their goal is
to determine whether these assignments preserve, or alter, the
classification outcome.

The task of deciding whether or not $f(\x_{\Bar{S}};\z_{S}) = f(\x)$
can, in turn, be formulated as an SOL formula, $\textit{SOL}_{\neg f}$, which
encodes that $f(\x_{\Bar{S}};\z_{S}) \neq f(\x)$ (or, again, the counting 
problem over that formula). If
$\textit{SOL}_{\neg f}$ is false, then the answer to the original problem is
affirmative; and otherwise, it is negative.

The relevant formula is fully quantified, in a manner that represents a 
specific explanation form. 
For example, contrastive reason queries check whether there \emph{exists} 
any assignment leading to a misclassification, whereas sufficient reason 
queries ask whether the classification stays constant for \emph{all} possible 
completions.
The goal is to eventually determine whether the formula is true or not, 
and equivalently --- whether the explanation is correct.

In the appendix, we show how each of the predefined 
explainability queries can be formalized using this notion, in which \emph{MSR} 
(Minimum Sufficient Reason)  
and \emph{MCR} (Minimum Change Required) are possible solutions to an 
underlying satisfiability query 
over $\textit{SOL}_{\neg f}$, and \emph{CC} (Count Completions) is the counting 
solution of 
$\#\textit{SOL}_{\neg f}$. This can also be extended to additional explanation 
forms.

\begin{definition}
  Let $\textit{SOL}_{\neg f}$ be an SOL formula encoding the query
  $f(\x_{\Bar{S}};\z_{S})\neq f(\x)$. We define an \emph{abstract
    query}, $\mathbf{Q}$, that receives $f$ and $\x$ as inputs, and
   answers whether $\textit{SOL}_{\neg f}$ is true.
  For the counting case, $\mathbf{Q}$ returns the counting of 
  $\#\textit{SOL}_{\neg f}$. 
\end{definition}

Next, we adjust this abstract query form to provide only socially
aligned explanations. This is performed by incorporating into the
formula the additional constraint $\pi(\x_{\Bar{S}};\z_{S})=1$, which
guarantees that any explanation that satisfies the query is also
in-distribution.
\begin{definition}
	Let $\textit{SOL}_{\neg f, \pi}$ be an SOL formula encoding the query 
	$f(\x_{\Bar{S}};\z_{S})\neq f(\x) \wedge 
	\pi(\x_{\Bar{S}};\z_{S})=1$. The respective aligned query, $\mathbf{Q}$, 
	receives as inputs $f$, 
	$\x$, and $\pi$, and answers whether $\textit{SOL}_{\neg f, \pi}$ is true.
	For the counting case, $\mathbf{Q}$ returns the counting of 
	$\#\textit{SOL}_{\neg f, 
	\pi}$. 
\end{definition}

Any of the aligned query forms mentioned in the previous section can be 
described as an abstract notion of this query as we show in the appendix. 
In essence, this abstract form captures various (logically expressible) 
explanation formulations, over which we can dismantle the effect of OOD 
counterfactuals.

By using this single, broader form of $\mathbf{Q}$, we are able
to prove general properties regarding socially aligned explainability queries, 
and deduce the complexity of interpreting these queries in various settings.

\section{The Complexity of Obtaining Socially Aligned Explanations}
\label{sec:social-alignment-complexity} 
\newcommand{\mcr}{\text{MCR}\xspace{}}
\newcommand{\mlp}{\text{MLP}\xspace{}}
\newcommand{\dt}{\text{DT}\xspace{}}
\newcommand{\fbdd}{\text{FBDD}\xspace{}}
\newcommand{\perceptron}{\text{Perceptron}\xspace{}}

\subsection{A General Framework}
To evaluate the computational complexity of interpreting a specific class of
models, denoted as $\mathcal{C}_\mathcal{M}$, it is useful to define
$Q(\mathcal{C}_\mathcal{M})$ as the computational problem represented by 
interpreting a set of
models within the class
$\mathcal{C}_\mathcal{M}$ with respect to an explainability query 
$Q$~\cite{BaMoPeSu20, BaAmKa24}. 
To illustrate, let us consider the
class of multi-layer perceptrons denoted as $\mathcal{C}_{\mlp}$. 
 MSR$(\mathcal{C}_{\mlp})$
is then the computational problem of
obtaining cardinally minimal sufficient reasons for an MLP, given an
input $\x$.

While this formalization is helpful for assessing the
interpretability of a specific model type, it does not consider the
underlying context and thus, it may produce socially misaligned explanations.

We revisit our running example, where our model $f$ represents a
decision tree. We further assume that the decision of whether $\x \in
\textbf{C}$ (or equivalently, whether $\x$ is in-distribution) is
learned by another model, e.g., a deep neural network. In this
scenario, the context indicator $\pi$ belongs to a different class
than $f$ (which is in $\mathcal{C}_\mathcal{M}$).
In such a case, we should pose a different type of question that
assesses the computational complexity of providing a socially aligned
explanation for an instance classified by $f$. Specifically, we need
to determine the computational complexity of interpreting a model
$f\in C_M$ while ensuring its alignment with a context indicator
function $\pi\in \mathcal{C}_{\pi}$. As mentioned earlier,  in many instances
(including our example), $\pi$ corresponds to a more expressive
function than $f$, potentially dominating the
  overall complexity. Therefore, we introduce the following notion that enables 
  us to assess the computational complexity of models in 
  $\mathcal{C}_{\mathcal{M}}$ with respect to a class of context indicators 
  $\mathcal{C}_{\pi}$.


\begin{definition}
	Given an explainability query $Q$, a class of prediction models 
	$\mathcal{C}_{\mathcal{M}}$, and a class of context indicators 
	$\mathcal{C}_{\pi}$, we define 
	\textbf{$Q(\mathcal{C}_{\mathcal{M}},\mathcal{C}_{\pi})$} as the 
	computational problem of $Q$ defined by the set of functions within 
	$\mathcal{C}_{\mathcal{M}}$, with respect to the contexts induced by the 
	functions of $\mathcal{C}_{\pi}$.
\end{definition}

For our running example, $Q(\mathcal{C}_{\dt},\mathcal{C}_{\mlp})$  denotes 
the computational complexity of some explainability query $Q$, given that our 
classification model is a decision tree and the OOD detector is a multi-layer 
perceptron. We note that, similarly to the previously studied evaluation of 
$Q(\mathcal{C}_{\mathcal{M}})$~\cite{BaMoPeSu20}, the formalization of 
$Q(\mathcal{C}_{\mathcal{M}},\mathcal{C}_{\pi})$ considers a ``worst-case'' 
scenario of the corresponding alignment, and not any parameter-specific 
configuration. 
This is captured by assessing the corresponding complexity with 
respect to a class of prediction models and a class of 
distribution 
indicators.




\subsection{The Complexity of $Q(\mathcal{C}_{\mathcal{M}},\mathcal{C}_{\pi})$}

We prove a connection between the complexity of calculating an aligned 
explanation $Q(\mathcal{C}_{\mathcal{M}},\mathcal{C}_{\pi})$, to the complexity 
of obtaining misaligned explanations of either $Q(\mathcal{C}_{\mathcal{M}})$ 
or $Q(\mathcal{C}_{\pi})$. This relation holds in a broad sense, as we prove it 
for our abstract query form $\mathbf{Q}$, defined in 
Sec.~\ref{general_query_form_section}. 



First, clearly, if $\mathbf{1}\in \mathcal{C}_{\pi}$ ($\mathbf{1}$ is a trivial 
function that accepts any possible input as in-context), then 
$\mathbf{Q}(\mathcal{C}_{\mathcal{M}},\mathcal{C}_{\pi})$ is polynomially 
reducible from $\mathbf{Q}(\mathcal{C}_{\mathcal{M}})$.
We note that $\mathbf{1}\in \mathcal{C}_{\pi}$ is a trivial request for any 
expressive class of context indicators, for example, assuming the existence of 
a neural network that always outputs $1$. 

\begin{theorem}
	\label{generalized_theorem_1}
	If $\mathbf{1}\in \mathcal{C}_{\pi}$ then 
	$\mathbf{Q}(\mathcal{C}_{\mathcal{M}})\leq_{p} 
	\mathbf{Q}(\mathcal{C}_{\mathcal{M}},\mathcal{C}_{\pi})$. 
\end{theorem}

This result is, of course, not surprising and a more interesting connection to 
explore is the less straightforward relation between 
$\mathbf{Q}(\mathcal{C}_{\mathcal{M}},\mathcal{C}_{\pi})$ and 
$\mathbf{Q}(\mathcal{C}_{\pi})$. 
We show that a similar result to the former can be obtained in this
case as well, provided that  $\mathcal{C}_{\pi}$ is 
\emph{symmetrically constructible} (given some $f\in \mathcal{C}_{\pi}$, we can 
construct in polynomial time $\neg f\in \mathcal{C}_{\pi}$); and that 
$\mathcal{C}_M$ is 
\emph{naively constructible} (given some $\x\in \mathbb{F}$, it holds that 
we can construct in polynomial time $\mathbf{1}_{\{\x\}}\in 
\mathcal{C}_{\mathcal{M}}$). A full formalization of these conditions is 
provided in the appendix. Later in this section, we also demonstrate that these 
constructions also hold for popular function classes, and provide 
model-specific instantiations of our framework. 
%

\begin{theorem}
	\label{generalized_theorem_2}
	If $\mathcal{C}_{\mathcal{M}}$ is symmetrically constructible and 
	$\mathcal{C}_{\pi}$ is naively constructible, then 
	$\mathbf{Q}(\mathcal{C}_{\pi})\leq_{p}\mathbf{Q}(\mathcal{C}_{\mathcal{M}},\mathcal{C}_{\pi})$.
	
\end{theorem}

Theorem~\ref{generalized_theorem_2} indicates that, given basic assumptions 
regarding the expressivity of $\mathcal{C}_{\mathcal{M}}$ and 
$\mathcal{C}_{\pi}$, it holds that the complexity of evaluating 
$\mathbf{Q}(\mathcal{C}_{\mathcal{M}},\mathcal{C}_{\pi})$, i.e., interpreting  
a model from $\mathcal{C}_{\mathcal{M}}$ with respect to a model from 
$\mathcal{C}_{\pi}$, 
for some explainability query $\mathbf{Q}$, is \emph{at least as hard} as 
interpreting $\mathbf{Q}(\mathcal{C}_{\pi})$, i.e., interpreting the OOD 
detector $\pi$. 
This is significant --- as in many cases $\mathcal{C}_{\pi}$, the class 
associated with the input distribution, is much more expressive than $C_M$, the 
class associated with the prediction model, and hence may be much harder to 
interpret. 

\medskip\noindent
\emph{Proof sketch.} The reduction exploits the naive constructibility
of $\mathcal{C}_{\mathcal{M}}$,
with the aim of
rendering obsolete the conjunct responsible for validating whether a subset is
contrastive. The reduction takes
advantage of the fact that $\pi\in \mathcal{C}_{\pi}$ is symmetrically
constructible in order to transform $\pi$ to validate the \emph{model}
instead of the indicated context. By employing this approach, it
becomes feasible to polynomially reduce any SOL formula representing
$\mathbf{Q}(\mathcal{C}_{\pi})$ to an equivalent SOL formula under the
formulation of
$\mathbf{Q}(\mathcal{C}_{\mathcal{M}},\mathcal{C}_{\pi})$. Consequently,
any decision or counting solution for the original SOL formula will be
tantamount to solving an equivalent SOL formula corresponding to a
query seeking socially aligned explanations.





\subsection{Model-Specific Framework Instantiations}
Next, we present specific results when focusing on 
FBDDs, Perceptrons, and MLPs. It is straightforward to show that these classes 
of models match our theoretical framework, as the following holds (and proven 
in our appendix):



\begin{proposition}
\label{general_model_form}
FBDDs, Perceptrons, and MLPs are all symmetrically constructible and naively 
constructible.
\end{proposition}


\mysubsection{Dominance of Interpreting MLPs.}  We prove that when dealing with 
complexity classes of explainability queries that are from the polynomial 
hierarchy (such as NP, $\Sigma^P_2$, etc.), the
complexity class associated with the MLP always dominates the overall
complexity. Hence, the \emph{exact} complexity class of
$Q(\mathcal{C}_{\mathcal{M}},\mathcal{C}_{\pi})$ when 
$\mathcal{C}_{\mathcal{M}}=\mathcal{C}_{\mlp}$ and/or 
$\mathcal{C}_{\pi}=\mathcal{C}_{\mlp}$, is equivalent to that of 
$Q(\mathcal{C}_{\mlp})$.
This claim holds for any class of polynomially computable functions.





\begin{theorem}
\label{theorem_mlp_always_wins}
Let $\mathcal{C}_{\mathcal{M}},\mathcal{C}_{\pi}$ be classes of 
polynomially computable functions such that $\mathcal{C}_\mathcal{M} = 
\mathcal{C}_{\mlp}$ or $\mathcal{C}_{\pi} = \mathcal{C}_{\mlp}$. If 
$\mathbf{Q}(\mathcal{C}_{\mlp})$ is $\mathcal{K}$-complete, where $\mathcal{K}$ 
is a complexity class of the polynomial hierarchy (or the class associated with 
its counting problem), then 
$\mathbf{Q}(\mathcal{C}_{\mathcal{M}},\mathcal{C}_{\pi})$ is also 
$\mathcal{K}$-complete.
\end{theorem}

The ``hardness'' part of Theorem~\ref{theorem_mlp_always_wins} is a direct 
consequence of Theorems~\ref{generalized_theorem_1} 
and~\ref{generalized_theorem_2}. However, when specifically considering MLPs, 
completeness also holds. The proof of this claim is relegated to the appendix, 
and is a result of the fact that any Boolean circuit can be polynomially 
reduced to an MLP~\cite{BaMoPeSu20}. This relation implies that the ``hardest'' 
possible 
complexity class in the polynomial hierarchy is always associated with the one 
for interpreting an MLP over $\mathbf{Q}$. Fig.~\ref{fig:complexityPlot} 
depicts the relations among different complexity classes, as derived from 
Theorems~\ref{generalized_theorem_1},~\ref{generalized_theorem_2}, 
and~\ref{theorem_mlp_always_wins}.




\begin{figure}[ht]
	\centering
	\includegraphics[width=0.5\textwidth]{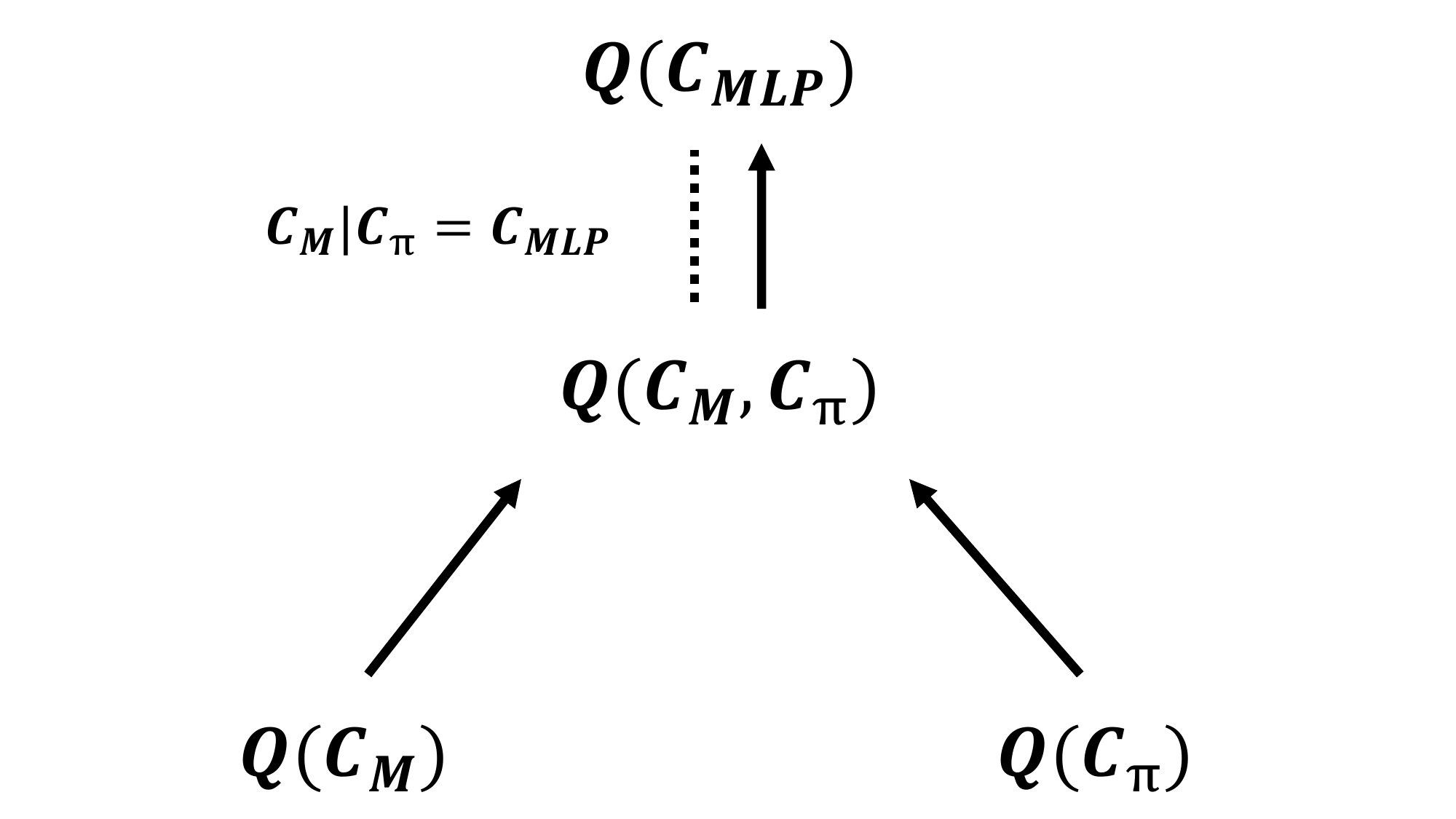}
	\caption{A visual illustration of Theorems~\ref{generalized_theorem_1}, 
	\ref{generalized_theorem_2}, and~\ref{theorem_mlp_always_wins}. Dashed 
	lines 
	depict that both queries are in the same complexity class, and are hard for 
	that class. Arrows are directed from the query with the ``easier'' 
	complexity class to the query with the ``harder'' complexity class.}
	\label{fig:complexityPlot}
\end{figure} %

\vspace{0.7cm}
In Table~\ref{table:results:VerificationResults}, we exemplify the 
aforementioned explainability queries (\emph{MCR}, \emph{MSR}, and \emph{CC}) 
and a specific scenario where $\mathcal{C}_{\mathcal{M}}$ is set to either 
$\mathcal{C}_{\fbdd}$ or $\mathcal{C}_{\perceptron}$, whereas the context 
indicator $\mathcal{C}_{\pi}$ is set to $\mathcal{C}_{\mlp}$ (this is the case 
of our running example, in which the OOD detection is performed using a more 
expressive model than the original classifier). 
Hence, Theorem~\ref{theorem_mlp_always_wins} implies that the complexity of 
solving the aligned query is primarily determined by the complexity involved in 
using an MLP, as summarized in Table~\ref{table:results:VerificationResults}.

\begin{table*}[t]
	\centering
	\caption{The computational complexity of 
		$\mathbf{Q}(\mathcal{C}_{\mathcal{M}})$
		and $\mathbf{Q}(\mathcal{C}_{\mathcal{M}},\mathcal{C}_{\pi})$ 
		with respect to various 
		explainability queries.} 
	\resizebox{0.65\textwidth}{!}{%
	\begin{tabular}{lcccccc}
		\toprule
		& & 
		\multicolumn{2}{c}{\textbf{$\mathcal{C}_{\mathcal{M}}=\mathcal{C}_{\text{FBDD}}$}}
		& 
		\multicolumn{2}{c}{\textbf{$\mathcal{C}_{\mathcal{M}}=\mathcal{C}_{\text{Perceptron}}$}}
		\\
		\cmidrule(lr){3-4} \cmidrule(lr){5-6}
		\textbf{} & & \textbf{Q($C_{M}$)} & 
		\textbf{Q($\mathcal{C}_{\mathcal{M}}$, $\mathcal{C}_{\text{MLP}}$)} & 
		\textbf{Q($\mathcal{C}_{\mathcal{M}}$)} & 
		\textbf{Q($\mathcal{C}_{\mathcal{M}}$, $\mathcal{C}_{\text{MLP}}$)} \\
		\midrule
		\textbf{MCR} & & \PComplexity & \NPCompleteComplexity & \PComplexity & 
		\NPCompleteComplexity \\[1ex]
		\textbf{MSR} & & \NPCompleteComplexity & \StoPCompleteComplexity & 
		\PComplexity & \StoPCompleteComplexity \\[1ex]
		\textbf{CC} & & \PComplexity & \countPCompleteComplexity & 
		\countPCompleteComplexity & \countPCompleteComplexity \\[1ex]
		\bottomrule
	\end{tabular}
}
	\label{table:results:VerificationResults}
\end{table*}

\section{``Self-Alignment'': Incorporating Social Alignment within a Single 
Model}
\label{sec:self-alignment}

Until now, we focused on the general scenario in which $f$ and $\pi$ are chosen 
from two \emph{different} model classes (for instance $f$ is a decision tree, 
and 
$\pi$ is a neural network). However, in some cases, $f$ and $\pi$ can be two 
models of the same type, i.e., from the same class. In this scenario, given a 
classifier and an OOD detector, both from the same class, practitioners 
might decide to train a single model that learns \emph{both} the prediction 
task and the alignment task.
More formally, we say that a single model class $\mathcal{C}$ is ``self-aligned'' when it 
is expressive enough to incorporate this dual procedure. This is demonstrated 
by the fact that given a model $f$ and a context indicator $\pi$, a new 
model $g$ can be efficiently constructed to show the alignment of $f$ with 
respect to the distribution indicated by $\pi$:

\begin{definition}
A class of models $\mathcal{C}$ is \emph{self-aligned} if for any $f,\pi \in 
\mathcal{C}$, and any inputs $\x$ and $I$, there exists a 
polynomially constructible function $g\in \mathcal{C}$, 
such that:
\begin{equation}
\begin{aligned}
	\langle f,\pi,\x, I \rangle \in \mathbf{Q}(\mathcal{C},\mathcal{C}) \iff 
	\langle g,\x, I \rangle \in \mathbf{Q}(\mathcal{C})
\end{aligned}
\end{equation}
\end{definition}


Intuitively, for any possible explainability query within $\mathbf{Q}$ 
(decision or counting), explanations of $f$, aligned by $\pi$, can be expressed 
by a single aggregated function $g$. Clearly, $g$ must be at least as 
expressive as 
the original models $f$ and $\pi$.
This raises the question of how expressive a class of models 
$\mathcal{C}$ should be, for it to be self-aligned.

\begin{theorem}
\label{expressivnes}
Given a class of models $\mathcal{C}$, if for any $f_{1}, f_{2} \in 
\mathcal{C}$, we can polynomially construct $g:=f_{1}[op]f_2\in \mathcal{C}$, 
for [op]$\in\{\wedge, \rightarrow \}$, then $\mathcal{C}$ is self-aligned.
\end{theorem}

Intuitively, classes of models that are capable of expressing the 
logical operators $\rightarrow$ and $\wedge$ are capable of ``capturing'' 
that a given explanation form is determined by its underlying 
distribution. The proof of this theorem is relegated to the appendix, and can 
be obtained by showing an equivalence between the two underlying formalizations.




If self-alignment implies that, given a prediction model $f$ and a context 
indicator $\pi$, we can attain a single aggregated model $g$ --- then clearly 
the computational complexity of interpreting $f\in C$ with respect to $\pi \in 
\mathcal{C}$ (i.e., the complexity of $\mathbf{Q}(\mathcal{C},\mathcal{C})$) is 
correlated to the complexity of interpreting $g\in \mathcal{C}$ (i.e., the 
complexity of $\mathbf{Q}(\mathcal{C})$). This can be demonstrated by the 
subsequent proposition: 

\begin{proposition}
\label{result-of-logical-containment}
If the conditions in Theorem~\ref{expressivnes} hold for a class of models 
$\mathcal{C}$, then 
$\mathbf{Q}(\mathcal{C},\mathcal{C})=_P\mathbf{Q}(\mathcal{C})$.
\end{proposition}



\subsection{Model-Specific Results}

We move on to analyze which of the aforementioned model classes incorporate 
self-alignment. First, we show that both FBDDs and MLPs are self-aligned, which 
is a result of their capability to polynomially express $\rightarrow$ and 
$\wedge$ relations within their class:

\begin{proposition}
FBDDs and MLPs are self-aligned, and hence, it follows that:
$\mathbf{Q}(\mathcal{C}_{\fbdd},\mathcal{C}_{\fbdd})=_P 
\mathbf{Q}(\mathcal{C}_{\fbdd})$ and 
$\mathbf{Q}(\mathcal{C}_{\mlp},\mathcal{C}_{\mlp})=_P 
\mathbf{Q}(\mathcal{C}_{\mlp})$.
\end{proposition}



%


However, in contrast to decision trees and neural networks, linear
classifiers lack the ability to capture the notion of self-alignment. 
It is important to note that a single Perceptron cannot inherently
represent the $\rightarrow$ and $\wedge$ relations over two other
Perceptrons. That said, it is worth emphasizing that this
observation alone does not conclusively establish their lack of self-alignment, 
as this condition is sufficient but not necessary. To rigorously prove the 
inability of Perceptrons to be self-aligned, we prove the 
subsequent proposition:

\begin{proposition}
\label{proposition:MCR_Perceptron_NPcomplete}
While the query MCR$(\mathcal{C}_{\perceptron})$ can be solved in polynomial 
time, the query MCR$(\mathcal{C}_{\perceptron},\mathcal{C}_{\perceptron})$ is 
NP-complete.
\end{proposition}
\medskip\noindent
\emph{Proof sketch.} Membership results from the fact that we can guess a 
subset of features $S$ and validate whether it is contrastive for $f$ and 
whether it is 
also in-distribution (by feeding it to $\pi$). For hardness, we reduce from 
\emph{SSP} (the k-subset-sum problem), which is a classic NP-complete 
problem. 
The reduction exploits the ranges of the Perceptrons of both $f$ and $\pi$ in 
order to bind the target sum $T$ of the subset, both from above and from below.

Building upon Proposition~\ref{proposition:MCR_Perceptron_NPcomplete}, we can 
deduce the following corollary 
(proved in the appendix):

\begin{theorem}
\label{perceptons_not_aligned}
Assuming that $P\neq NP$, the class $\mathcal{C}_{\perceptron}$ is not 
self-aligned.
\end{theorem}

These findings underscore a crucial aspect concerning the interpretability of 
Perceptrons. While producing explanations pertaining to them can be achieved 
with low computational complexity (providing further evidence of their 
interpretability), they are not self-aligned.
Consequently, obtaining \emph{aligned} explanations using Perceptrons 
necessitates the adoption of a more sophisticated model, that is expressive 
enough to incorporate social alignment --- and this, in turn, can significantly 
increase the overall complexity of their interpretation.


\section{Related Work}
\label{sec:related-work}
This work continues a line of research that focuses on \emph{Formal 
XAI}~\cite{ignatiev2020towards,waldchen2021computational,arenas2022computing,ArBaBaPeSu21,ignatiev2019relating,
 BaKa23, BaAmCoReKa23}.
Prior studies have already investigated the explanation forms  
that were 
analyzed within our 
work~\cite{ArBaBaPeSu21,waldchen2021computational,arenas2022computing,ArBaBaPeSu21},
 including 
	sufficiency-based explainability queries 
	(\emph{MSR})~\cite{ignatiev2019abduction,marques2020explaining}, 
	contrastive/counterfactual-based queries 
 (\emph{MCR})~\cite{shih2018formal,ignatiev2020contrastive}, and counting-based 
 queries (\emph{CC})~\cite{darwiche2002knowledge}. 
Other work~\cite{GoRu22} defined formal notions of 
sufficient and 
contrastive reasons under specific contexts and suggested ways to compute them 
on 
a wide range of models~\cite{yu2022eliminating}. However, these explanation 
forms were 
not analyzed with respect to their overarching computational complexity. Closer 
to ours is the work of Cooper et al.~\cite{cooper2023abductive} 
which analyzes 
different properties (including the computational complexity) of 
sufficiency-based explanations under logical constraints. 
We also acknowledge the work of Arenas et al.~\cite{ArBaBaPeSu21}, which 
describes a general logic-based explanation form, similar to our abstract 
query form. While their work focuses on explanations of first-order logic 
forms for decision queries, our approach is more expressive, 
encompassing 
second-order logic forms that incorporate both decision-based and 
counting-based explanations.

Another line of research examines the computational complexity of
obtaining Shapley value-based
explanations~\cite{arenas2021tractability, van2022tractability,
  marzouk2024tractability}, where alignment with respect to a given 
  distribution is vital~\cite{sundararajan2020many}. Specifically, Van den 
  Broeck et al.~\cite{van2022tractability} identify a complexity gap in 
  interpreting Shapley values when considering fully factorized or \emph{Naive 
  Bayes-modeled distributions}.

In some cases, the term ``sufficient reason'' is also defined as an 
\emph{abductive explanation}~\cite{ignatiev2019abduction} and correlates with 
the notion of a \emph{prime implicant} for a Boolean 
classifier~\cite{darwiche2002knowledge}.
The \emph{CC} query is associated with probabilistic notions of explainability, 
by correlating the precision of the explanation with the number of possible 
input completions~\cite{ribeiro2018anchors,waldchen2021computational}. A 
similar notion, formally known as a \emph{$\delta$-relevant 
set}~\cite{izza2021efficient,waldchen2021computational}, focuses on 
bounding this specific portion.

The dependency of explanations on OOD assignments has been studied
extensively~\cite{zaidan2007using,sundararajan2017axiomatic,fong2017interpretable,hooker2019benchmark,kim2020interpretation,hsieh2020evaluations,yi2020information,sundararajan2020many}.
Specifically, many heuristic-based tools and frameworks have been proposed 
for dealing with the OOD counterfactual problem in 
model explainability. 
These include marginalizing the prediction of the model over possible 
counterfactual 
assignments~\cite{zintgraf2017visualizing,kim2020interpretation,yi2020information},
 sampling points in the proximity of the original 
input~\cite{chang2018explaining,sanyal2021discretized,ribeiro2018anchors}, as 
well as \emph{counterfactual training}~\cite{HaXiBa21,vafa2021rationales} ---
a method that, similarly to adversarial training~\cite{ZhAlMa22}, seeks to 
robustify models to OOD counterfactuals. Other work focuses on mitigating the 
effect of OOD assignments on the computation of Shapley 
values~\cite{sundararajan2020many,janzing2020feature,taufiq2023manifold}. 
In spite of these notable accomplishments, the theoretical analysis of the OOD 
counterfactual problem with respect to its computational complexity has yet to 
be thoroughly examined.

\section{Conclusion}
\label{sec:conclusion}

Computational complexity theory stands as a potential avenue to formally assess 
the interpretability of various ML models. Prior research examined this by 
considering two main factors: the model type and the explanation form. We 
claim that a third and important factor should be taken into consideration --- 
the underlying distribution over which the explanation is computed. To achieve 
this goal, we generalize existing explainability queries and show how a unified 
form can describe the desired social alignment requirement for any explanation 
form under our second-order logic formalization. Moreover, we 
present a framework for assessing the computational complexity of these queries 
and demonstrate that, for a broad range of model types and query forms, 
providing socially aligned explanations is as hard as interpreting a model 
designed to detect OOD inputs. As OOD detection is known to be substantially 
difficult, such models may often require more expressive capacity than the 
original classification models, significantly impacting the overall complexity 
of model 
interpretation. 
Finally, we provide an analysis of the required capacity of models to 
inherently produce aligned explanations without using an external OOD detector. 
We hope that our work serves as a foundation for a deeper mathematical 
understanding of the interpretability pertaining to various ML models.

\section{Limitations and Future Work}
\label{sec:limitations}

Our framework can be extended along several different
axes. First and foremost, we note that assuming the existence of a 
context indicator $\pi$ for identifying OOD inputs is highly non-trivial. 
Previous work, both theoretical and practical, has highlighted the challenges 
associated with obtaining such an OOD 
detector~\cite{FaLiLuDoHaLi22,PeYi22,HsShJiKi20,SeAlGoSlNuLu19,BeXiMaZhLiXuZh20}.
However, it is important to emphasize that our framework does not necessarily 
assume the complete accuracy or correctness of such a classifier. 
Instead, $\pi$ can be viewed as a function that provides an 
\emph{approximation} of 
the underlying context $\mathbf{C}$. Therefore, future research endeavors 
could 
center around evaluating the computational complexity of specific 
approximations tailored to particular contexts of interest. While these 
approximations may only offer a \emph{partially} guaranteed solution to the 
alignment issue, they may still exhibit an improved complexity overall.

Other limitations correspond to similar
(non-aligned) approaches for analyzing the computational complexity of 
obtaining explanations~\cite{BaMoPeSu20, waldchen2021computational, 
BaAmKa24}. Firstly, our analysis considers only a worst-case scenario that 
may 
change under various parameter-specific configurations. Secondly, the natural 
subjectivity of interpretability makes it challenging to analyze the 
computational complexity of interpreting a model in a single ``correct'' way. 
To address this issue, theoretical frameworks define various explainability 
queries and evaluate them separately. We regard our proof for a wide range of 
explainability queries $\mathbf{Q}$ (the abstract query form) as potential 
evidence that the shared 
characteristics among different types of explainability queries can be utilized 
to offer more generalized assessments. 
%
%

Finally, we highlight that our study primarily concentrates on an OOD detector 
$\pi(\x)$, which classifies each input as either in-distribution or OOD, rather 
than on the input distribution $p_{\theta}(\x)$ itself. This approach is due to 
the strictly formal nature of the explanations we investigate; an explanation 
is either valid or not, necessitating a definitive categorization of the 
presence or absence of each input. In contrast, probabilistic explanation 
forms, such as $\delta$-relevant sets~\cite{waldchen2021computational, 
izza2021efficient} or Shapley values~\cite{lundberg2017unified, 
sundararajan2020many}, are defined in relation to the distribution itself and 
can also be assessed based on the computational complexity of obtaining them. 
For instance, a recent study by Marzouk et al.~\cite{marzouk2024tractability} 
explores the computational complexity of calculating Shapley values within 
Markovian distributions. Future research can focus on expanding the strictly 
formal explanation framework discussed here to include 
probabilistic 
explanation forms as well, where complexity assessments would focus directly on 
the input distribution $p_{\theta}(\x)$ rather than on the OOD detector 
$\pi(\x)$. 
Other, broader future work can explore the relation between the 
computational complexity of generating explanations (our current focus) and the 
complexity of the explanations themselves. This can be achieved using various 
tools, such as Kolmogorov complexity. We also cover additional 
extensions of our framework in the appendix.




\section*{Acknowledgments}
This work was partially funded by the European Union (ERC, VeriDeL, 101112713). 
Views and opinions expressed are however those of the author(s) only and do not 
necessarily reflect those of the European Union or the European Research 
Council Executive Agency. Neither the European Union nor the granting authority 
can be held responsible for them.
The work of Amir was further supported by a scholarship from the Clore Israel 
Foundation.



{
\bibliography{references.bib}
} 

\clearpage 
\appendix

\setcounter{definition}{0}
\setcounter{proposition}{0}
\setcounter{theorem}{0}
\setcounter{lemma}{0}


%

\begin{center}\begin{huge} Appendix\end{huge}\end{center}

\noindent{The appendix contains definitions, formalizations, and proofs that 
were mentioned throughout the paper:}

\newlist{MyIndentedList}{itemize}{4}
\setlist[MyIndentedList,1]{%
label={},
noitemsep,
leftmargin=0pt,
}

\begin{MyIndentedList}

\item \textbf{Appendix~\ref{appendix:general_query_form}} describes the 
abstract query form.


\item  \textbf{Appendix~\ref{appendix:models}} 
describes the specific model types, and the 
universal model properties. 

\item \textbf{Appendix~\ref{appendix:model_specific_results}} contains an 
analysis of the model-specific properties. %

\item  \textbf{Appendix~\ref{appendix:main_proofs}} includes the proofs for the 
theorems and propositions mentioned in the paper.

\item  \textbf{Appendix~\ref{appendix:extensions}} includes possible extensions 
of our theoretical framework.

\end{MyIndentedList}

\section{Abstract Query Form}
\label{appendix:general_query_form}

We present a comprehensive analysis of the abstract query form $\mathbf{Q}$ 
discussed in our paper, providing a more detailed explanation of our process. 

\subsection{The Misaligned Case}
Initially, we introduce the query form for the ``misaligned'' case, referring 
to the abstract query form $\mathbf{Q}$ that does not consider the context 
indicator $\pi$ as an input. We formulate this specific scenario and 
demonstrate its applicability in generalizing the various discussed 
explainability queries: \emph{MSR}, \emph{CC}, and \emph{MCR}, all in their 
misaligned versions. Subsequently, we proceed to outline the consequential 
query for the aligned scenario, where $\pi$ is taken into account, dismantling 
the effect of any OOD counterfactual. We then reiterate how the explainability 
queries of \emph{MSR}, \emph{CC}, and \emph{MCR}, when considered in their 
aligned forms, are all specific instances of the abstract query $\mathbf{Q}$.

Let $\psi_{\neg f}(\x,S,\z)$ 
denote the following conjunct:
\begin{equation}
\begin{aligned}
\psi_{\neg f} \coloneqq [f(\x_{\Bar{S}};\z_{S})\neq f(\x)]
\end{aligned}
\end{equation}

Let $\textit{SOL}_{\neg f}$ denote an SOL formula that includes $\psi_{\neg 
f}$, 
where the 
variable $f$ is exclusively present in $\psi_{\neg f}$. In other words, $f$ is 
not found in any other conjunct of $\textit{SOL}_{\neg f}$ apart from 
$\psi_{\neg f}$. We 
denote the (misaligned) version of $\mathbf{Q}$ as any explainability query 
that takes $f$, $x$, and $I$ as inputs, where $I$ represents a set of 
additional arbitrary inputs. The output of $\mathbf{Q}$ is a satisfying 
solution to $\textit{SOL}_{\neg f}$ or the counting of $\#\textit{SOL}_{\neg 
f}$. 
We are now able to formulate the abstract (misaligned) query form $\mathbf{Q}$:

\vspace{0.5em} 

\noindent\fbox{%
\parbox{\columnwidth}{%
	\mysubsection{(Misaligned) $\mathbf{Q}$ (Abstract Query Form)}:
	
	\textbf{Input}: Model $f$, input $\x$, and input $I$.
	
	\textbf{Output}: a \yes{} or \no{} answer, to whether $\textit{SOL}_{\neg 
	f}$ holds, 
	or 
	the number of assignments of $\textit{SOL}_{\neg f}$.
}%
}
\vspace{0.5em} 

It is important to highlight that the values of $S$ and $\z$ are implicitly 
present in $\psi_{\neg f}$. These values can either be included as part of the 
input $I$, or they can be explicitly defined within $\textit{SOL}_{\neg f}$. 
Now, we demonstrate how the aforementioned explainability queries (in their 
misaligned form) can be precisely formulated as specific instances of 
$\mathbf{Q}$. We start by illustrating this for the relatively simpler 
scenarios of \emph{MCR} and \emph{CC}:

\vspace{0.5em} 

\noindent\fbox{%
\parbox{\columnwidth}{%
\mysubsection{(Misaligned) MCR (Minimum Change Required)}:

\textbf{Input}: Model $f$, input $\x$, and input $I\coloneqq\langle k\rangle$.

\textbf{Output}: \yes{}, if $\textit{SOL}_{\neg f}\coloneqq \exists S\subseteq 
(1,\ldots,n) 
\ \ \exists (\z\in \mathbb{F}) \ \psi_{\neg f}(\x,S,\z)\wedge |S|\leq k$ is 
satisfiable, and \no{} otherwise.
}%
}

\vspace{0.5em} 

\noindent\fbox{%
\parbox{\columnwidth}{%
\mysubsection{(Misaligned) CC (Count Completions)}:

\textbf{Input}: Model $f$, input $\x$, and input $I\coloneqq\langle S\rangle$.

\textbf{Output}: The number of assignments of $\textit{SOL}_{\neg f}\coloneqq 
\exists 
(\z\in \mathbb{F}) \ \psi_{\neg f}(\x,S,\z)$.
}%
}
\vspace{0.5em} 

Once again, it is worth noting that in these particular scenarios, the value of 
$S$ is derived from a subset of the input $I$ for the \emph{CC} query, while 
for \emph{MCR}, it is defined within $\textit{SOL}_{\neg f}$. Now, we proceed 
to demonstrate how the sufficiency-based explainability query (\emph{MSR}) can 
also be acquired. In this case, $\textit{SOL}_{\neg f}$ corresponds to the 
\emph{negation} of $\psi_{\neg f}$:

\vspace{0.5em} 

\noindent\fbox{%
\parbox{\columnwidth}{%
\mysubsection{(Misaligned) MSR (Minimum Sufficient Reason)}:

\textbf{Input}: Model $f$, input $\x$, and input $I\coloneqq\langle k\rangle$.

\textbf{Output}: 
\yes{}, if $\textit{SOL}_{\neg f}\coloneqq \exists S\subseteq (1,\ldots,n) \ \ 
\neg \exists 
(\z\in \mathbb{F}) \ \psi_{\neg f}(\x,\Bar{S},\z)\wedge |S|\leq k$ is 
satisfiable, and \no{} otherwise.
}%
} 
\vspace{0.5em} 

It is straightforward to show that this formalization of \emph{MSR} is 
equivalent to its predefined versions, since:

\begin{equation}
\begin{aligned}
\forall (\z\in \mathbb{F}) \quad [f(\x_{S};\z_{\Bar{S}})= f(\x)] \iff \\ \neg 
\exists (\z\in \mathbb{F}) \quad [f(\x_{\Bar{S}};\z_{S})\neq f(\x)] 
\end{aligned}
\end{equation}

In other words, a subset $S$ is sufficient to determine a prediction $f(\x)$ if 
and only if there does not exist any assignment to the complementary $\bar{S}$ 
that is contrastive. This leads us to the fact that there exists a sufficient 
reason of size $k$ if and only if there exists some subset $S$ of size $k$ such 
that no possible assignment to $\bar{S}$ is contrastive. 

We note that when evaluating the complexity of the misaligned abstract 
explainability query $\mathbf{Q}$, we consider it in relation to a single class 
of models. For instance, $\mathbf{Q}(\mathcal{C}_{\mathcal{M}})$ describes the 
complexity of obtaining a (misaligned) explainability query for a model $f\in 
\mathcal{C}_{\mathcal{C}}$ using the (misaligned) abstract query form 
$\mathbf{Q}$. Unlike the aligned version, we do not input two families of 
functions since the context indicator $\pi$ is not included as part of the 
input for these queries.

\subsection{The Aligned Case}
We now proceed to describe the \emph{aligned} version of the abstract query 
form. In this case, we aim to construct a similar abstract query form with the 
additional requirement of neutralizing the influence of any OOD 
counterfactuals. This is obtained by adding an additional constraint, namely 
$[\pi(\x_{\bar{S}};\z_{S})= 1]$. To formally define this, we introduce 
$\psi_{\neg f, \pi}(\x,S,\z)$ as follows:

\begin{equation}
\begin{aligned}
\psi_{\neg f, \pi} \coloneqq [f(\x_{\bar{S}};\z_{S})\neq f(\x)]\wedge 
[\pi(\x_{\bar{S}};\z_{S}) = 1]
\end{aligned}
\end{equation}

Similarly, we define $\textit{SOL}_{\neg f,\pi}$ as any SOL formula that 
includes 
$\psi_{\neg f,\pi}$, where $f$ and $\pi$ are exclusively included within 
$\psi_{\neg f,\pi}$. In other words, $f$ and $\pi$ are not present in any 
conjunct of $\textit{SOL}_{\neg f,\pi}$ except for $\psi_{\neg f,\pi}$. We 
denote the 
(aligned) version of $\mathbf{Q}$ as an explainability query that takes $f$, 
$x$, $\pi$, and $I$ as inputs, where $I$ represents an arbitrary set of 
additional inputs. The output of $\mathbf{Q}$ is a solution (either decision or 
counting), over $\textit{SOL}_{\neg f, \pi}$. Therefore, the abstract query 
form, in this 
case, can be expressed as follows:

\vspace{0.5em} 

\noindent\fbox{%
\parbox{\columnwidth}{%
\mysubsection{$\mathbf{Q}$ (Abstract Query Form)}:

\textbf{Input}:Model $f$, input $\x$, context indicator $\pi$, and input $I$.

\textbf{Output}: a \yes{} or \no{} answer, to whether $\textit{SOL}_{\neg f, 
\pi}$ 
holds, or 
the number of assignments of $\textit{SOL}_{\neg f, \pi}$.
}%
}
\vspace{0.5em} 

We briefly illustrate how this abstract query form encompasses all the 
previously defined explainability queries, including the aligned versions of 
\emph{MSR}, \emph{MCR}, and \emph{CC}. Once again, it is straightforward to 
show that \emph{MCR} and \emph{CC} are instances of this abstract query form 
(this time, in the \emph{aligned} version):

\vspace{0.5em} 

\noindent\fbox{%
\parbox{\columnwidth}{%
\mysubsection{MCR (Minimum Change Required)}:

\textbf{Input}: Model $f$, input $\x$, context indicator $\pi$, and input 
$I\coloneqq\langle k\rangle$.

\textbf{Output}: \yes{}, if $\textit{SOL}_{\neg f, \pi}\coloneqq \exists 
S\subseteq 
(1,\ldots,n) \ \ \exists (\z\in \mathbb{F}) \ \psi_{\neg f, \pi}(\x,S,\z)\wedge 
|S|\leq k$ is satisfiable, and \no{} otherwise.
}
}
\vspace{0.5em} 

\vspace{0.5em} 
\noindent\fbox{%
\parbox{\columnwidth}{%
\mysubsection{CC (Count Completions)}:

\textbf{Input}: Model $f$, input $\x$, context indicator $\pi$, and input 
$I\coloneqq\langle S\rangle$.

\textbf{Output}: The number of assignments of $\textit{SOL}_{\neg f, 
\pi}\coloneqq \exists 
(\z\in \mathbb{F}) \ \psi_{\neg f, \pi}(\x,S,\z)$.
}%
}
\vspace{0.5em} 

Similarly to the misaligned case, the aligned version of the sufficiency-based 
query (\emph{MSR}) can be obtained as follows: 

\vspace{0.5em} 

\noindent\fbox{%
\parbox{\columnwidth}{%
\mysubsection{MSR (Minimum Sufficient Reason)}:

\textbf{Input}: Model $f$, input $\x$, context indicator $\pi$, and input 
$I\coloneqq\langle k\rangle$.

\textbf{Output}: 
\yes{}, if $\textit{SOL}_{\neg f, \pi}\coloneqq \exists S\subseteq (1,\ldots,n) 
\ \ \neg 
\exists (\z\in \mathbb{F}) \ \psi_{\neg f, \pi}(\x,\Bar{S},\z)\wedge |S|\leq k$ 
is satisfiable, and \no{} otherwise.
}%
} 
\vspace{0.5em} 

The equivalence between these specific instances and the predefined 
explainability queries holds in these particular scenarios. This is due to the 
following:

\begin{equation}
\begin{aligned}
\forall (\z\in \mathbb{F}) \quad [\pi(\x_{S};\z_{\Bar{S}})=1 \to 
f(\x_{S};\z_{\Bar{S}})= f(\x)] \iff \\
\neg \exists (\z\in \mathbb{F}) \quad [f(\x_{\Bar{S}};\z_{S})\neq f(\x)]\wedge 
[\pi(\x_{\Bar{S}};\z_{S})=1] 
\end{aligned}
\end{equation}


Recall that in the case of the aligned version of $\mathbf{Q}$, the underlying 
computational complexity is evaluated by considering two classes of models: 
$\mathcal{C}_{\mathcal{M}}$ for the classification model and 
$\mathcal{C}_{\pi}$ for the context indicator. For instance, 
$\mathbf{Q}(\mathcal{C}_{\mathcal{M}}, \mathcal{C}_{\pi})$ represents the 
computational complexity of obtaining an explainability query for models $f\in 
\mathcal{C}_{\mathcal{M}}$ with respect to the context indicators $\pi \in 
\mathcal{C}_{\pi}$ and an abstract query form $\mathbf{Q}$.

\section{Model Types and Universal Properties}
\label{appendix:models}

\subsection{Model Types}
\label{appendix:model_types}
Next, we provide a full description of the models that are taken into account 
within our work.

\mysubsection{Binary Decision Diagram (BDD).} A BDD~\cite{Le59} is a graphical 
representation of a Boolean function $f: \mathbb{F}  \to \{0,1\}$, realized by 
a directed, acyclic graph, for which: 
\begin{inparaenum}[(i)]
\item each internal node $v$ (i.e., non-sink nodes) corresponds to a single 
feature $(1,\ldots,n)$;
\item each internal node $v$ has precisely two output edges, representing the 
values $\{0,1\}$ assigned to $v$;
\item each leaf corresponds to either a true, or false, label; and 
\item each variable appears at most once, along any given path $\alpha$ within 
the BDD.
\end{inparaenum}

Hence, every path $\alpha$ from the root node to a leaf, corresponds to a 
specific input assignment $\x\in\mathbb{F}$, with $f(\x)$ matching the value of 
the leaf of the relevant path $\alpha$. Following previous 
conventions~\cite{BaMoPeSu20,huang2021efficiently,huang2023feature,arenas2022computing},
 we regard the size $|f|$ of the BDD to be the total number of edges. We focus 
on the popular hypothesis class of ``Free BDDs'' (FBDDs), in which different 
paths may have various orderings of the input variables $\{1,\ldots,n \}$. 



\mysubsection{Multi-Layer Perceptron (MLP).} Given a set of $t$ weight matrices 
$W^{(1)},\ldots,W^{(t)}$, $t$ bias vectors $b^{(1)},\ldots,b^{(t)}$ and $t$ 
activation functions $f^{(1)},\ldots,f^{(t)}$, a Multi-Layer Perceptron 
(MLP)~\cite{GaDo98,RaGhEtAm16} $f$, with $t-1$ hidden layers ($h^{j}$ for 
$j\in\{1,\ldots, t-1\}$) and a single output layer ($h^{t}$), is recursively 
defined based on the following series of functions:

\begin{equation}
h^{(j)} \coloneqq \sigma^{(j)}(h^{(j-1)}W^{(j)} + b^{(j)}) \quad (j \in 
\{1,\ldots,t\})
\end{equation}

$f$ outputs the value of the function $f \coloneqq h^{(t)}$, and 
$h^{(0)}\coloneqq \x \in \{0,1\}^n$ corresponds to the input of the model. The 
weight matrices and biases are defined by a series of positive values 
$d_{0},\ldots,d_{t}$ representing the dimensions of their inputs. In addition, 
we assume that all the weights and biases (learned during training) have 
rational values, i.e., $W^{(j)}\in \mathbb{Q}^{d_{j-1}\times d_{j}}$ and 
$b^{(j)}\in \mathbb{Q}^{d_{j}}$. Notice that due to our focus on \emph{binary} 
classifiers over $\{0,1\}^n$, then it holds that: $d_{0} = n$ and $d_{t} = 
1$. Furthermore, we consider the popular $ReLU(x)=\max(0,x)$ activation 
function. The last 
activation of MLPs is typically a sigmoid function, but since we are only 
interested in post-hoc interpretations, we can equivalently, without loss of 
generality, consider 
the last activation to correspond to the \emph{step} function:
\begin{equation}
\label{eq:step-definition}
\begin{aligned}
\text{step}(y) =
\begin{cases}
	1, & \begin{aligned} &y > 0  \end{aligned} \\
	0, & \begin{aligned} &y \leq 0 \end{aligned}
\end{cases}
\end{aligned}
\end{equation}

\mysubsection{Perceptron.} A Perceptron~\cite{RaReHe03} is an MLP with a single 
layer (i.e., $t=1$): $f(\x) = \sigma(\langle W,\x\rangle+b)$, for 
$W\in\mathbb{Q}^{n\times d_{1}}$ and $b\in\mathbb{Q}$. Hence, without loss of 
generality, for a 
Perceptron $f$ it holds that: 

\begin{equation}
f(\x)=1 \iff \langle W,\x\rangle+b > 0
\end{equation}


\subsection{Universal Properties}

Next, we provide the precise formalization for the universal properties over 
$\mathcal{C}_{\mathcal{M}}$ and $\mathcal{C}_{\pi}$ that were mentioned within 
our study. These are that $\mathcal{C}_{\pi}$ is \emph{symmetrically 
constructible}, whereas $\mathcal{C}_{\mathcal{M}}$ is \emph{naively 
construcatble}.

\begin{definition}
A class of functions $\mathcal{C}$ is \textbf{symmetrically constructible} if 
given a model $f\in \mathcal{C}$, then $\neg f\in \mathcal{C}$ can be 
constructed in polynomial time.
\end{definition}

\begin{definition}
A class of functions $\mathcal{C}$ is \textbf{naively constructible} if for any 
value $\x\in \mathbb{F}$, then $\mathbf{1}_{\{\x\}}\in \mathcal{C}$ can be 
constructed in polynomial time.
\end{definition}

\section{Model-Specific Properties}
\label{appendix:model_specific_results}

As mentioned above, our propositions and theorems are based on the universal 
properties formulated in Section~\ref{appendix:models} of the appendix. These 
qualities include symmetric constructability and naive constructability. In 
this section, we 
illustrate how these properties indeed hold for the particular models discussed 
in our work, namely FBDDs, Perceptrons, and MLPs. We emphasize that these are 
only particular illustrations, and that these properties can be proven to hold 
for a broader range of hypothesis classes. First, we recall 
Proposition~\ref{appendix:general_model_form}:

\begin{proposition}
\label{appendix:general_model_form}
FBDDs, Perceptrons, and MLPs are all symmetrically constructible and naively 
constructible.
\end{proposition}

\noindent
To this end, we prove the following lemmas.

\begin{lemma}
\label{fbdds_symetrically_constructible}
The class $\mathcal{C}_{\fbdd}$ is naively constructible and symmetrically 
constructible.
\end{lemma}

It is straightforward to show this in the following manner:
\begin{inparaenum}[(i)]
\item given an FBDD $f$ we can construct $\neg f$ by duplicating $f$ and 
negating all leaf nodes $v$ in the duplicated diagram; and
\item given an input $\x \in \mathbb{F}$ we can simply construct an FBDD $f$ 
with a single accepting path $\alpha$ matching the assignment of 
$\x$.\end{inparaenum} 

\begin{lemma}
\label{MLPs_symetrically_constructible}
The class $\mathcal{C}_{\mlp}$ is naively constructible and symmetrically 
constructible.
\end{lemma}


MLPs can also be constructed symmetrically and naively in a straightforward 
manner. 
First, we state that for every MLP $f$, we can construct, in linear time, an 
equivalent MLP $f'$, such that the weights and biases are integers (this can be 
achieved by multiplying the values by the lowest common denominator, as done 
in~\cite{BaMoPeSu20}). Next, for the bias in the last layer, we also add 
$-0.5$. 
This procedure guarantees that: (i) for 
every input $x\in\{0,1\}^n$, it holds that $f(x)=f'(x)$, i.e., the new MLP $f'$ 
is equivalent to $f$; and (ii) there is no binary input $x$ 
such that for $f'=step(h'^{(t-1)}W'^{(j)} + b'^{(t)})$ it holds that 
$(h'^{(t-1)}W'^{(j)} + b'^{(t)})(x)=0$, i.e., no input $x$ is \emph{exactly} on 
the decision 
boundary of $f'$ (as all linear combinations of integers --- remain integers, 
and the single bias is not an integer).
Next, symmetric constructability for $f'$ (and hence, for $f$) is acquired as 
follows. We 
can construct $\neg f'$ by negating the weights of the last layer $h'^t$ 
(setting $h'^t_{i}=-h'^t_{i}$ for all $i$) and negating the bias of the output 
layer $b'^t$. Since the last layer contains a single \emph{step} function, 
negating 
the corresponding weights (and bias) will result in a flipped classification. 

To show naive constructability, we make use of the following 
Lemma~\cite{BaMoPeSu20}:

\begin{lemma}
\label{lemma_boolean_circuit}
Given a Boolean circuit $B$, we can construct, in polynomial time, an MLP 
$f_{B}$, which induces an equivalent Boolean function relative to $B$.
\end{lemma}

Hence, as a direct corollary, it is possible to polynomially construct an MLP 
that corresponds to the Boolean circuit representing: $x_1 \wedge x_2 \ldots 
\wedge x_n$.




\begin{lemma}
\label{perceptrons_symetrically_constructible}

The class $\mathcal{C}_{\perceptron}$ is naively constructible and 
symmetrically 
constructible.
\end{lemma}

The symmetric construction proved for MLPs also holds directly for 
Perceptrons (by negating the weights of $h^1$ as well as the bias $b^1$). 
Given an input $\x\in \mathbb{F}$, naive constructability can be achieved
by constructing a model where the corresponding single hidden layer $h^{1}$ is 
weighted such that $h^{1}_i\coloneqq w^{+}$ for $x_i=1$ and $h^{1}_i \coloneqq 
w^{-}$ for $x_i = 
0$, for some user-defined value $w$. The single bias term $b^1$ is set to $b^1 
\coloneqq 
- [\sum_{1\leq i\leq n}(h_{i}^{1}\cdot x_{i})]+0.5$. The intuition behind 
this construction is that it maximizes the contribution of the particular input 
$\x$ while rendering negative values for any other input in 
$\mathbb{F}\setminus \x$. 
An illustration of this construction is provided in 
Fig.~\ref{fig:perceptronIndicator}. Also, we observe that this construction 
clearly serves as valid proof for the naive constructability of MLPs, but this 
was already trivially derived from the properties discussed in the previous 
section.

\section{Main Theorem Proofs}
\label{appendix:main_proofs}
In this section, we prove all the theorems and propositions presented in the 
main text.

\subsection{The Complexity of Obtaining Socially Aligned Explanations}

First, we provide the proofs for Theorems~\ref{generalized_theorem_1} 
and~\ref{generalized_theorem_2}, as discussed in the main text. More 
specifically: 




\begin{figure}[ht]
	\centering
	\includegraphics[width=0.45\textwidth]{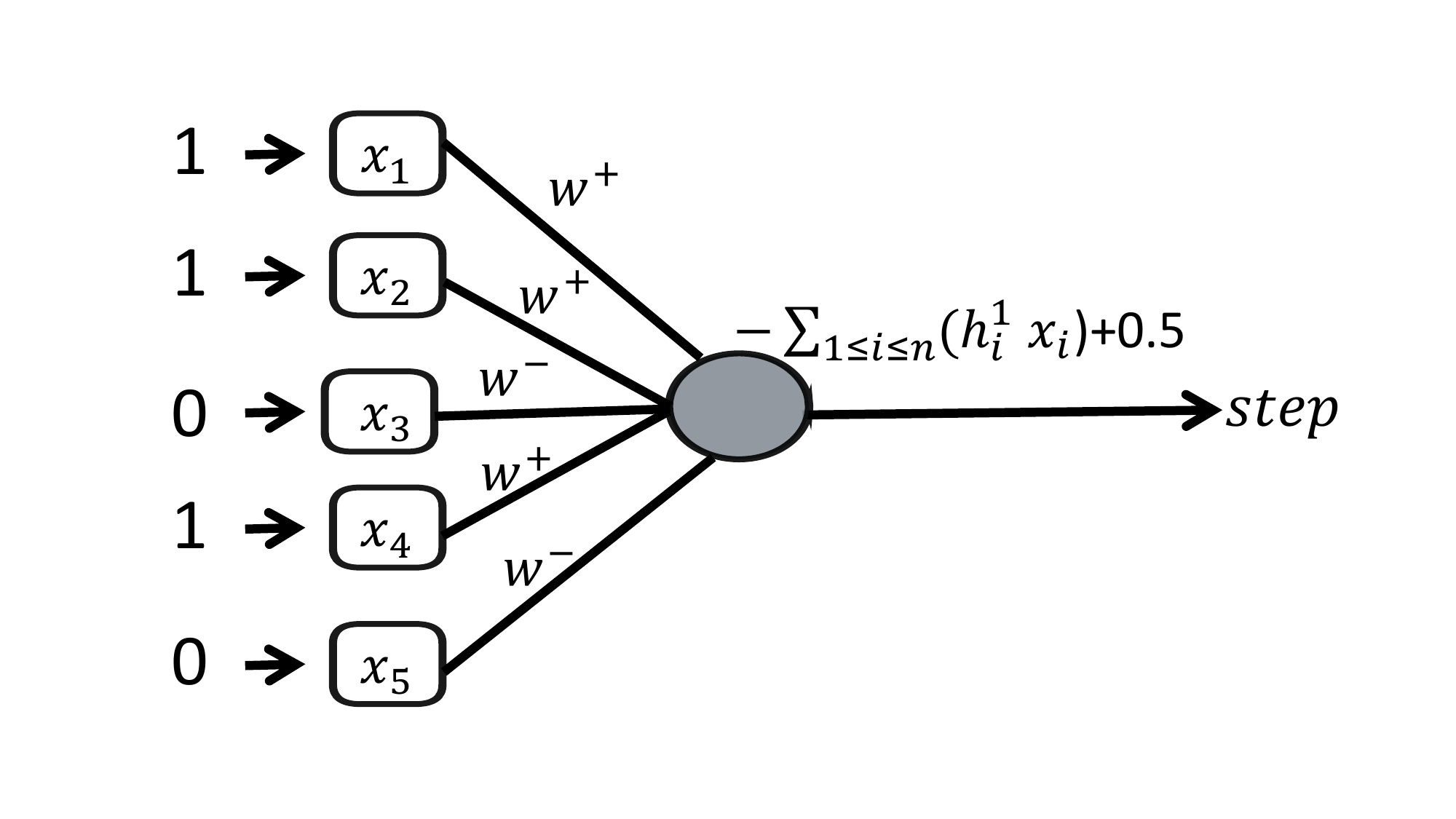}
	\caption{An illustration of the naive constructability of a Perceptron 
		model, 
		indicating the value $[1,1,0,1,0]$, and $w^+=1$, $w^-=(-1)$. The bias 
		term is 
		[$- \sum_{1\leq i\leq n}(h^{1}_i\cdot x_i)] + 0.5=(-3)+0.5=(-2.5)$.}
	\vspace{20pt}
	\label{fig:perceptronIndicator}
\end{figure}
\begin{theorem}
\label{appendix:generalized_theorem_1}
If $\mathbf{1}\in \mathcal{C}_{\pi}$ then 
$\mathbf{Q}(\mathcal{C}_{\mathcal{M}})\leq_{p} 
\mathbf{Q}(\mathcal{C}_{\mathcal{M}},\mathcal{C}_{\pi})$. 
\end{theorem}

\begin{proof}
The proof is straightforward since given some $\langle f,\x,I \rangle$ the 
reduction can simply encode and return $\langle f,\x,\mathbf{1} ,I\rangle$. 
Clearly, it holds that: $\langle f,\x,I \rangle \in 
\mathbf{Q}(\mathcal{C}_{\mathcal{M}}) \iff \langle f,\x,\mathbf{1},I\rangle \in 
\mathbf{Q}(\mathcal{C}_{\mathcal{M}},\mathcal{C}_{\pi})$, which concludes the 
correctness of the reduction.
\end{proof}

\begin{theorem}
\label{appendix:generalized_theorem_2}
If $\mathcal{C}_{\mathcal{M}}$ is symmetrically constructible and 
$\mathcal{C}_{\pi}$ is naively constructible, then 
$\mathbf{Q}(\mathcal{C}_{\pi})\leq_{p}\mathbf{Q}(\mathcal{C}_{\mathcal{M}},\mathcal{C}_{\pi})$.

\end{theorem}
Given some $\langle f_1, \x, I\rangle$: the reduction checks whether $f_1$ is a 
valid encoding of a function in $\mathcal{C}_{\pi}$. If not, it returns an 
invalid encoding. If so, it constructs the negation function $\neg f_1$ (based 
on our assumptions, this can be computed in polynomial time). Then, the 
reduction computes $f_1(\x)$ and constructs $\mathbf{1}_{\{\x\}}\in 
\mathcal{C}_{\mathcal{M}}$  (also in polynomial time). If $f_1(\x)=1$, the 
reduction returns $\langle f_2=\mathbf{1}_{\{\x\}}, \pi_2=\neg f_1, \x, 
I\rangle$, and if $f_1(\x)=0$, it returns $\langle f_2=\mathbf{1}_{\{\x\}}, 
\pi_2=f_1, \x, I\rangle$.

Let $Q_1$ denote the SOL formula that corresponds to the solution of 
$\mathbf{Q}(\mathcal{C}_{\pi})$ and let $Q_2$ denote the SOL formula that 
corresponds to the solution of 
$\mathbf{Q}(\mathcal{C}_{\mathcal{M}},\mathcal{C}_{\pi})$. Let us denote 
$\psi_{\neg f_{1}}$ as the aforementioned conjunct (see 
Sec.~\ref{appendix:general_query_form} of the appendix) that corresponds to 
$Q_1$, and by $\psi_{\neg f_{2}, \pi}$ the conjunct that corresponds to $Q_2$. 

Assume $\langle f_1, \x, I\rangle \in \mathbf{Q}(\mathcal{C}_{\pi})$. Since in 
this case it holds that $f_1\in \mathcal{C}_{\pi}$, then: 
\begin{align*}
	\psi_{\neg f_{1}} = [f_1(\x_{S};\z_{\Bar{S}})\neq f_1(\x)] 
\end{align*}

For $\phi_2 \coloneqq [\pi_2(\x_{S};\z_{\Bar{S}})=1]$ it holds that:

\begin{align*}
	\psi_{\neg f_{2}, \pi} = [f_2(\x_{S};\z_{\Bar{S}})\neq f_2(\x)] \wedge 
	[\pi_2(\x_{S};\z_{\Bar{S}}) = 1] \\
	=[f_{2}(\x_{S};\z_{\Bar{S}})\neq f_{2}(\x)] \wedge
	\phi_2 \\
	= [\mathbf{1}_{\{\x\}}(\x_{S};\z_{\Bar{S}})\neq \mathbf{1}_{\{\x\}}(\x)] 
	\wedge
	\phi_2\\
	=[\mathbf{1}_{\{\x\}}(\x_{S};\z_{\Bar{S}})\neq 1] \wedge
	\phi_2\\
\end{align*}




Assume that $f_1(\x)=1$. In this case, the reduction sets $\pi_2$ to $\neg 
f_1$, and thus:

\begin{align*}
	\phi_2 = [\pi_2(\x_{S};\z_{\Bar{S}})=1] = [\neg f_1(\x_{S};\z_{\Bar{S}})=1] 
	=  \\
	[f_1(\x_{S};\z_{\Bar{S}})\neq1]=
	[f_1(\x_{S};\z_{\Bar{S}})\neq1]\wedge [f_1(\x)=1]
\end{align*}

Where the last encoded conjunct $f_1(\x)=1$ is a tautology under this scenario. 
Overall, we get that:
\begin{align*}
	\psi_{\neg f_{2}, \pi} = [\mathbf{1}_{\{\x\}}(\x_{S};\z_{\Bar{S}})\neq 1] 
	\wedge [f_1(\x_{S};\z_{\Bar{S}})\neq1] \\ \wedge [f_1(\x)=1]
	= [f_1(\x_{S};\z_{\Bar{S}})\neq1] = \psi_{\neg f_{1}}
\end{align*}

This means that $Q_1$ and $Q_2$ are equivalent and hence any solution for SOL 
or \#SOL will be equivalent to 
$\mathbf{Q}(\mathcal{C}_{\mathcal{M}},\mathcal{C}_{\pi})$ and 
$\mathbf{Q}(\mathcal{C}_{\pi})$. Thus, $\langle f_1,  \x, I\rangle \in 
\mathbf{Q}(\mathcal{C}_{\pi}) \iff \langle f_2, \x, \pi_2, I\rangle \in 
\mathbf{Q}(\mathcal{C}_{\mathcal{M}},\mathcal{C}_{\pi})$.

Assume that $f_1(\x)=0$. In this case, the reduction sets $\pi_2$ to $f_1$ and 
thus:

\begin{align*}
	\phi_2 = [\pi_2(\x_{S};\z_{\Bar{S}})=1] = [f_1(\x_{S};\z_{\Bar{S}})=1]
	\\= [f_1(\x_{S};\z_{\Bar{S}})\neq0] =  
	[f_1(\x_{S};\z_{\Bar{S}})\neq0]\wedge [f_1(\x)=0]
\end{align*}


Overall, we again obtain that:
\begin{align*}
	\psi_{\neg f_{2}, \pi} = [\mathbf{1}_{\{\x\}}(\x_{S};\z_{\Bar{S}})\neq 1] 
	\wedge [f_1(\x_{S};\z_{\Bar{S}})\neq0] \\ \wedge [f_1(\x)=0]
	= [f_1(\x_{S};\z_{\Bar{S}})\neq0] = \psi_{\neg f_{1}}
\end{align*}


Hence, again it holds that $Q_1$ and $Q_2$ are equivalent, and from the same 
reason stated above, it thus holds that $\langle f_2, \x, \pi_2, I\rangle \in 
\mathbf{Q}(\mathcal{C}_{\mathcal{M}},\mathcal{C}_{\pi})$.

Now, assume $\langle f_1, \x, I\rangle \not\in \mathbf{Q}(\mathcal{C}_{\pi})$. 
In this case, the reduction initially checks the validity of the encoding, 
which includes that of $f\in \mathcal{C}_{\pi}$. Hence, we are only left to 
check the cases where $f\in \mathcal{C}_{\pi}$ but $\langle f, \x, \pi, 
I\rangle \not\in \mathbf{Q}(\mathcal{C}_{\pi})$. This implies that the SOL was 
unsatisfiable or, if \textbf{$Q$} is a counting query, that \#SOL returned an 
incorrect count. Since the previous result demonstrated that $Q_1$ and $Q_2$ 
are equivalent under the assumption that $f\in \mathcal{C}_{\pi}$, any 
assignment to $Q_1$ will hold if and only if it holds to $Q_2$. Consequently, 
we can 
conclude that $\langle f_2, \x, \pi_2, I\rangle \not\in 
\mathbf{Q}(\mathcal{C}_{\mathcal{M}},\mathcal{C}_{\pi})$.

\begin{theorem}
\label{appendix:theorem_mlp_always_wins}
Let $\mathcal{C}_{\mathcal{M}},\mathcal{C}_{\pi}$ be classes of 
polynomially computable functions such that $\mathcal{C}_{\mathcal{M}} = 
\mathcal{C}_{\mlp}$ or $\mathcal{C}_{\pi} = \mathcal{C}_{\mlp}$. If 
$\mathbf{Q}(\mathcal{C}_{\mlp})$ is $\mathcal{K}$-complete, where $\mathcal{K}$ 
is a complexity class of the polynomial hierarchy (or its associated counting 
class), then $\mathbf{Q}(\mathcal{C}_{\mathcal{M}},\mathcal{C}_{\pi})$ is also 
$\mathcal{K}$-complete.
\end{theorem}
\medskip\noindent
\emph{Proof.} 
The complexity classes within the polynomial hierarchy consist of the classes 
$\Sigma_{k}^p$ and $\Pi_{k}^p$ for all $k$. The class $\Sigma_{k}^p$ is defined 
as all languages $L$ such that there exists a polynomial time Turing machine 
$M$ and polynomials $q_1,\ldots, q_k$ such that:

\begin{equation}
\begin{aligned}
	x\in L \iff \exists y_1 \forall y_2, \ldots (\exists/\forall)y_k, |y_i|\leq 
	q_i(|X|) \\ \wedge \ \ M(x,y_1,\ldots,y_k)=1
\end{aligned}
\end{equation}

On the other hand, $\Pi_{k}^p$ is defined respectively by alternating 
quantifiers $\forall \exists,\ldots$, and hence $\Pi_{k}^p= \{L|\Bar{L}\in 
\Sigma_{k}^p\}$. In the context of SOL, we define $SOL^k$ as a formula of 
the form:

\begin{equation}
\begin{aligned}
	\exists X_1 \forall X_2, \ldots (\exists/\forall)X_k, \phi(X_1,\ldots,X_k)
\end{aligned}
\end{equation}

\noindent
where $\phi$ is a quantifier-free FOL formula over 
$(X_1,\ldots,X_k)$. 
Extending Fagin's theorem~\cite{Fa74} shows that a solution to an SOL formula 
with $k$ alternating quantifiers (starting with $\exists$) is 
$\Sigma_{k}^p$-complete. We note that this holds in the case of finite 
structures, 
which is our setting, as for any $n$ we have a finite number of inputs. 
Since any SOL formula can be written as an SOL formula consisting of 
alternating 
quantifiers $\exists,\forall,\ldots$ or 
$\forall,\exists,\ldots$, then by extending Fagin's theorem~\cite{Fa74}, we can 
conclude that each SOL formula is associated with a class in 
the polynomial hierarchy (this holds in the case of finite structures, which 
is indeed our case, as we focus on discrete inputs). For example 
$\exists\forall SOL$ is 
$\Sigma_{2}^p$-complete and $\forall\exists SOL$ is $\Pi_{2}^p$-complete 
(again, as in 
our case the formula includes models that are restricted to finite inputs). 
We 
note that for the counting case, each one of these complexity classes has a 
corresponding associated counting class, for example, the number of satisfying 
assignments for $\exists SOL$ or $\exists\forall SOL$.

Recall that each query $\mathbf{Q}(C_M)$ is associated with an SOL formula 
$\textit{SOL}_{\neg f}$ and each aligned query $\mathbf{Q}(C_M,C_{\pi})$ is 
associated 
with an SOL formula $\textit{SOL}_{\neg f, \pi}$. Both of these formulas can be 
written in 
their Prenex Normal Form ${SOL}^k_{\neg f}$ and ${SOL}^k_{\neg f, \pi}$. Since 
the 
prefixes of both these formulas are equivalent, then both of them are 
associated with some complexity class $K'$ in the polynomial hierarchy (or its 
corresponding counting class), due to the extension of Fagin's Theorem, which 
as mentioned, 
holds for our case (as we focus on discrete inputs of a finite size).

Clearly, since $\mathcal{C}_{\mathcal{M}}$ and $\mathcal{C}_{\pi}$ are 
classes of polynomially computable functions, then by definition 
$\mathbf{Q}(\mathcal{C}_{\mathcal{M}},\mathcal{C}_{\pi})\in \mathcal{K}'$ and 
$\mathbf{Q}(\mathcal{C}_{\mathcal{M}})\in \mathcal{K}'$. More specifically, 
this means that $\mathbf{Q}(\mathcal{C}_{\mlp})\in \mathcal{K}'$ as well. Next, 
to prove hardness, we will make use of Lemma~\ref{lemma_boolean_circuit}, as 
well as Karp's reduction.

Karp's reduction implies that any quantified propositional 
formula $\exists 
X_1 \forall 
X_2, \ldots (\exists/\forall)X_k, \psi(X_1,\ldots,X_k)$, for a quantifier-free 
Boolean formula $\psi$ (over 
$(X_1,\ldots,X_k)$) is $\Sigma_{k}^p$-complete. 
In addition, as  Lemma~\ref{lemma_boolean_circuit} indicates that an arbitrary 
Boolean formula $\psi$ can be translated to an equivalent MLP in polynomial 
time, 
it holds that $\mathbf{Q}(\mathcal{C}_{\mlp})$ is $\mathcal{K'}$-hard. As we 
have also shown that $\mathbf{Q}(\mathcal{C}_{\mlp})$ is in $\mathcal{K'}$, we 
deduce that 
$\mathbf{Q}(\mathcal{C}_{\mlp})$ is $\mathcal{K'}$-complete.


Overall, we get that 
$\mathbf{Q}(\mathcal{C}_{\mlp})$ is $\mathcal{K}$-complete for some class in 
the polynomial hierarchy, and also $\mathcal{K}'$-complete. Each language is 
complete only for one class in the hierarchy (or otherwise, the hierarchy 
collapses), and thus $\mathcal{K}=\mathcal{K}'$.

Now, we get that $\mathbf{Q}(\mathcal{C}_{\mathcal{M}},\mathcal{C}_{\pi})\in 
\mathcal{K}'=\mathcal{K}$. Since we know that 
$\mathcal{C}_{\mathcal{M}}=\mathcal{C}_{\mlp}$ or 
$\mathcal{C}_{\pi}=\mathcal{C}_{\mlp}$, and since MLPs are both symmetrically  
constructible and naively constructible 
(Lemma~\ref{MLPs_symetrically_constructible}) then as a consequence of 
Theorems~\ref{appendix:generalized_theorem_1} 
and~\ref{appendix:generalized_theorem_2} we get that 
$\mathbf{Q}(\mathcal{C}_{\mathcal{M}},\mathcal{C}_{\pi})$ is also 
$\mathcal{K}$-hard. We deduce that 
$\mathbf{Q}(\mathcal{C}_{\mathcal{M}},\mathcal{C}_{\pi})$ is 
$\mathcal{K}$-complete. 

\medspace
\noindent
\textbf{Note.}
In order to rely on Fagin's thorem~\cite{Fa74}, we must assume a finite 
structure. This holds in the case of \emph{discrete} inputs, but not for any 
general SOL encoding (which in fact, may even be undecidable).

\subsection{``Self-Alignment'': Incorporating Social Alignment within a Single 
Model}

In the next subsection, we elaborate on the general, and model-specific, 
proofs for our findings regarding the self-alignment property of a given model 
class $\mathcal{C}$.
First, we reiterate the definition of self-alignment:

\setcounter{definition}{3}

\begin{definition}
A class of models $\mathcal{C}$ is \textbf{self-aligned} if for any $f,\pi \in 
\mathcal{C}$, and any inputs $\x$ and $I$, there exists a 
polynomially constructible function $g\in \mathcal{C}$, 
such that:
\begin{equation}
	\begin{aligned}
		\langle f,\pi,\x, I \rangle \in \mathbf{Q}(\mathcal{C},\mathcal{C}) 
		\iff \langle g,\x, I \rangle \in \mathbf{Q}(\mathcal{C})
	\end{aligned}
\end{equation}
\end{definition}

\begin{theorem}
\label{appendix:expressivnes}
Given a class of models $\mathcal{C}$, if for any $f_{1}, f_{2} \in 
\mathcal{C}$, we can polynomially construct $g:=f_{1}[op]f_2\in \mathcal{C}$, 
for [op]$\in\{\wedge, \rightarrow \}$, then $\mathcal{C}$ is self-aligned.
\end{theorem}

\medskip\noindent
\emph{Proof.} 
Based on the assumptions on $\mathcal{C}$, given any two models $f,\pi \in 
\mathcal{C}$, we can construct, in polynomial time, a function $g$ that encodes 
a logical relation between the original classifier $f$ and the context 
indicator $\pi$. We define $g\in \mathcal{C}$ based on the original 
classification $f(\x)$.
Our reduction defines slightly different functions $g \in \mathcal{C}$, 
depending on whether the original classification is $f(\x)=1$ or 
$f(\x)=0$.

In the case of $f(\x)=1$, we define $g\in \mathcal{C}$ as follows:
\begin{equation}
\begin{aligned}
	g\coloneqq [f \lor \neg \pi] \in \mathcal{C}. 
\end{aligned}
\end{equation}
Next, we note that for the given input $\x$ it holds that:
\begin{equation}
\begin{aligned}
	g (\x) = [f \lor \neg \pi] (\x) = f(\x) \lor \pi(\x) =  1 \lor \neg \pi(\x) 
	= 1 
\end{aligned}
\end{equation}

It also holds that:
\begin{equation}
\begin{aligned}
	\psi_{\neg f, \pi} 
	\iff \\ 
	[f(\x_{S};\z_{\Bar{S}})\neq f(\x)=1] \wedge 
	[\pi(\x_{S};\z_{\Bar{S}})=1] 
	\iff \\
	[f(\x_{S};\z_{\Bar{S}}) = 0] \wedge 
	[\pi(\x_{S};\z_{\Bar{S}})=1]  
	\iff \\
	[\neg f(\x_{S};\z_{\Bar{S}}) = 1] \wedge 
	[\pi(\x_{S};\z_{\Bar{S}})=1]  
	\iff \\ 
	[(\neg f \wedge \pi) (\x_{S};\z_{\Bar{S}}) = 1]
	\iff \\
	[\neg(\neg f \wedge \pi) (\x_{S};\z_{\Bar{S}}) = 0]  
	\iff \\ 
	[(f \lor \neg \pi) (\x_{S};\z_{\Bar{S}}) = 0]  \iff \\
	[g(\x_{S};\z_{\Bar{S}}) = 0 \neq 1 = g(\x)] 
	\iff \psi_{\neg g} 
\end{aligned}
\end{equation}

In the case of $f(\x)=0$, we define $g\in \mathcal{C}$ as follows:
\begin{equation}
\begin{aligned}
	g\coloneqq [f \wedge \pi] \in \mathcal{C}. 
\end{aligned}
\end{equation}

And hence, if $f(\x)=1$:
\begin{equation}
\begin{aligned}
	\langle f,\x,\pi,I\rangle \in \mathbf{Q}(\mathcal{C},\mathcal{C}) \iff
	\langle g,\x,I \rangle \in \mathbf{Q}(\mathcal{C})
\end{aligned}
\end{equation}

In the case that for the given input $\x$ it holds that:
\begin{equation}
\begin{aligned}
	g (\x) = [f \wedge \pi] (\x) = f(\x) \wedge \pi(\x) =  0 \wedge \pi(\x) = 0 
\end{aligned}
\end{equation}

It also holds that:
\begin{equation}
\begin{aligned}
	\psi_{\neg f, \pi} 
	\iff \\
	[f(\x_{S};\z_{\Bar{S}})\neq f(\x)=0] \wedge 
	[\pi(\x_{S};\z_{\Bar{S}})=1] 
	\iff \\
	[f(\x_{S};\z_{\Bar{S}}) = 1] \wedge 
	[\pi(\x_{S};\z_{\Bar{S}})=1]
	\iff \\
	[(f \wedge \pi) (\x_{S};\z_{\Bar{S}}) = 1]  \iff \\
	[g(\x_{S};\z_{\Bar{S}}) = 1 \neq 0 = g(\x)] 
	\iff \psi_{\neg g}
\end{aligned}
\end{equation}

Hence, for all cases, it holds that:
\begin{equation}
\begin{aligned}
	\langle f,\x,\pi,I\rangle \in \mathbf{Q}(\mathcal{C},\mathcal{C}) \iff 
	\langle g,\x,I \rangle \in \mathbf{Q}(\mathcal{C})
\end{aligned}
\end{equation}

Thus, we conclude that  $\mathcal{C}$ is self-aligned, and hence  
Theorem~\ref{appendix:expressivnes} is proven.

\begin{proposition}
\label{appendix:result-of-logical-containment}
If the conditions in 
Theorem~\ref{expressivnes} hold for a class of models $C$, then 
$\mathbf{Q}(\mathcal{C},\mathcal{C})=_P\mathbf{Q}(\mathcal{C})$.
\end{proposition}

\medskip\noindent
\emph{Proof.}
Based on the previous reduction, we can deduce that if the 
conditions in Theorem~\ref{expressivnes} hold for a model class $\mathcal{C}$, 
then $\mathcal{C}$ is self-aligned, i.e., given $f,\pi \in \mathcal{C}$ we can 
polynomially construct a function $g \in \mathcal{C}$, such that:

\begin{align*}    
\langle f,\pi,\x, I \rangle \in \mathbf{Q}(\mathcal{C},\mathcal{C}) \iff 
\langle g,\x, I \rangle \in \mathbf{Q}(\mathcal{C})
\end{align*}


Hence: 
\begin{align*}  
\mathbf{Q}(\mathcal{C},\mathcal{C}) \leq_{p} \mathbf{Q}(\mathcal{C})
\end{align*}

In addition, from Theorems~\ref{appendix:generalized_theorem_1} 
and~\ref{appendix:generalized_theorem_2}, it holds that:

\begin{align*}  
\mathbf{Q}(\mathcal{C}) \leq_{p} \mathbf{Q}(\mathcal{C},\mathcal{C})
\end{align*}

Finally, it is straightforward to conclude that:
\begin{align*}  
\mathbf{Q}(\mathcal{C},\mathcal{C})=_P\mathbf{Q}(\mathcal{C})
\end{align*}

\begin{proposition}
\label{appendix:fbdd_and_mlp_self_alignment}
FBDDs and MLPs are self-aligned, and hence, it follows that:
$\mathbf{Q}(\mathcal{C}_{\fbdd},\mathcal{C}_{\fbdd})=_P 
\mathbf{Q}(\mathcal{C}_{\fbdd})$ and 
$\mathbf{Q}(\mathcal{C}_{\mlp},\mathcal{C}_{\mlp})=_P 
\mathbf{Q}(\mathcal{C}_{\mlp})$.
\end{proposition}

\begin{lemma}
The class $\mathcal{C}_{\fbdd}$ is self-aligned.
\end{lemma}
\medskip \noindent
\emph{Proof.}
We rely on Proposition~\ref{appendix:expressivnes} and show the corresponding 
encodings of $\wedge$ and $\rightarrow$ for $\mathcal{C}_{\fbdd}$ (a sufficient 
condition for the self-alignment of the class). More specifically, given two 
functions $f_t,f_s\in C_{FBDD}$, over $\{0,1\}^n$, we can polynomially 
construct a new function $f_k\in \mathcal{C}_{\fbdd}$, such that $f_k = f_t 
\wedge f_s$. 

As a first step, we show how, given $f_t$ and $f_s$, we can construct some 
general decision diagram $f':=f_t\rightarrow f_s$ or $f':=f_t\wedge f_s$. In 
the second step of this process, we explain how we can reduce $f'$ to some 
$f_k\in \mathcal{C}_{\fbdd}$.

\begin{lemma}
\label{first_lemma_fbdd}
Given some $f_t,f_s\in \mathcal{C}_{\fbdd}$, it is possible to polynomially 
construct the Boolean functions $f':=f_t\rightarrow f_s$ and $f':=f_t\wedge 
f_s$.
\end{lemma}
\medskip\noindent
\emph{Proof.}
Notice that we do not require that $f'\in \mathcal{C}_{\fbdd}$, as we later 
show, $f'$ will be a general form of a decision diagram, and the encoding 
of $f'$ will be of size $O(f_s \cdot f_t)$. For $f':=f_s\wedge f_t$ we perform 
the following steps:

\begin{enumerate}
\item Create a single copy of $f_t$.

\item 
For each leaf node labeled ``1'' in 
$f_t$, delete the leaf and connect its predecessor to the root of a copy of 
$f_s$.

\item 
We note that all leaf nodes labeled ``0'' in $f_t$ are left unchanged. 

\end{enumerate}


It is hence straightforward to show that $\forall \ \z\in \mathbb{F}$:

\begin{align*}
f'(\z) = 1 \iff f_t(\z) = 1 \wedge  f_s(\z) = 1
\end{align*}

Next, we demonstrate how, given two functions $f_t,f_s\in \mathcal{C}_{\fbdd}$ 
over $\{0,1\}^n$, we can polynomially construct a function $f'$, satisfying 
$f_k = f_t \rightarrow f_s$, again using an encoding of size $O(f_t\cdot 
f_s)$.


Specifically, our construction includes the following steps:

\begin{enumerate}
\item Since the class $\mathcal{C}_{\fbdd}$ is symmetrically constructible 
(Proposition~\ref{fbdds_symetrically_constructible}), we can polynomially 
construct a function $f':=\neg f_t\in \mathcal{C}_{\fbdd}$. 
\item We construct a new function $f_k\in \mathcal{C}_{\fbdd}$ by deleting all 
``0'' leaf nodes of $f'$, and adding edges between the predecessors of the 
``0'' leaves, and the root of a copy of $f_s$.
\end{enumerate}

It is now straightforward to show that $\forall \z\in \mathbb{F}$:

\begin{align*}
f_k(\z)=0 
\iff \\
\neg f_t(\z)=0 \wedge f_s(\z)=0 \iff \\
f_t(\z)=1 \wedge f_s(\z)=0 
\end{align*}

Hence, $f' = f_t \rightarrow f_s$. Fig.~\ref{fig:combineFbdds} depicts this 
construction. We emphasize that the decision diagram $f'$ is \emph{not 
necessarily an FBDD}. This is because the same variables can repeat themselves 
more than once within $f'$. We later prove how any $f'$ can be reduced to 
some $f_k \in \mathcal{C}_{\fbdd}$.


\begin{figure}[ht]
\centering
\includegraphics[width=0.48\textwidth]{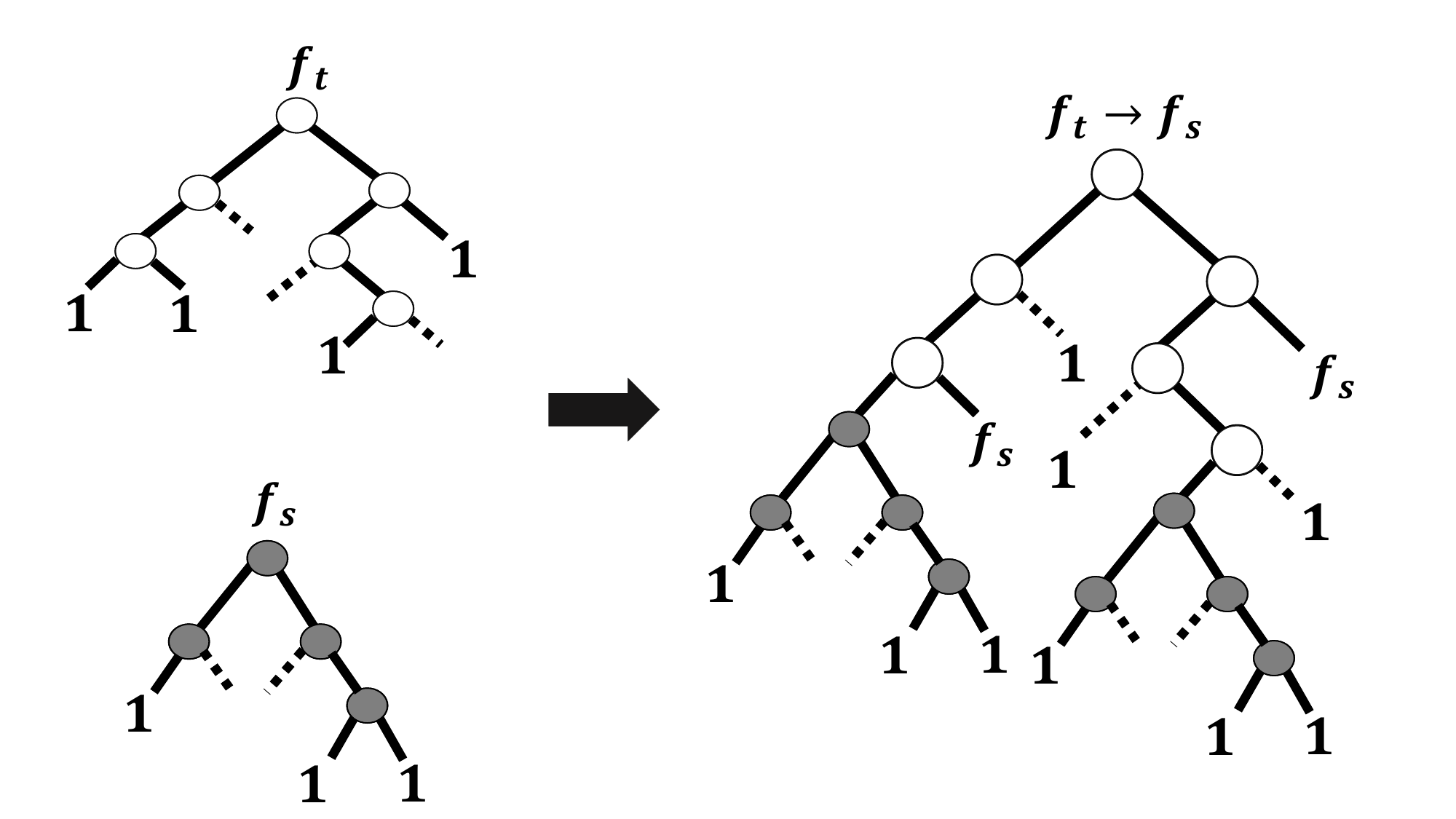}
\caption{An illustration of the polynomial construction $f':=f_t \rightarrow 
f_s$, relying on $f_t, f_s\in \mathcal{C}_{\fbdd}$. For $f_s$ and $f_t$ the 
dashed lines represent paths that end with a ``0'' leaf node, while solid lines 
represent paths that end with a ``1'' leaf node.}
\vspace{20pt}
\label{fig:combineFbdds}
\end{figure}




\begin{lemma}
Given the encoding of $f'$ from Lemma~\ref{first_lemma_fbdd}, we can 
polynomially reduce $f'$ to some $f_k\in \mathcal{C}_{\fbdd}$.
\end{lemma}
\medskip\noindent
\emph{Proof.}
As mentioned earlier, $f'$ is not necessarily an FBDD since it may contain 
various repeated variables within its diagram. We now show how each $f'$ can 
be reduced to some $f_k\in \mathcal{C}_{\fbdd}$.


We define each path $\alpha$ as a concatenation of two subpaths $\alpha 
\coloneqq [\alpha_{t} ; \alpha_{s}]$, each corresponding the the relevant path 
in $f_s$ (or $f_t$) accordingly. Since each node $v$ corresponds to some input 
feature $i\in (1,\ldots,n)$, we denote $x_v(v)$ as a function that maps $v$ to 
its corresponding feature. We use the common conventions of 
$parent(v)$, $left(v)$, and $right(v)$.

We now describe the following recursive algorithm for reducing $f'$ to some 
$f_k\in \mathcal{C}_{\fbdd}$.

\begin{enumerate}

\item We traverse on all corresponding paths $\alpha$, starting from the root 
node and traversing downwards. The algorithm does not change $f'$ as long as we 
are on the $\alpha_{t}$ part of the traversion.
\item 
Reaching a node $v_s\in \alpha_{s}$, if for all $v_i\in \alpha_t$ it holds that 
$x_v(v_i) \neq x_v(v_s)$, then we recursively continue traversing both left and 
right. Intuitively, this means that the feature corresponding to $v_s$ was not 
part of the decision path of $\alpha_t$.
\item 
Otherwise, we reach some node $v_{s} \in \alpha_s$ such that there exists a 
$v_i\in \alpha_t$ in which $x_v(v_i) = x_v(v_s)$. In this case, we delete $v_s$ 
from $f'$. We now need to connect the parent of $v_s$ with either the left or 
right child of $v_s$ within $f'$. Assume that for 
$x_v(v_{i+1})=right(x_v(v_i))$, or in other words $v_i$ leads to a right turn 
in the path $\alpha$. In this case, we connect the parent of $v_s$ with the 
right child of $v_s$. If the opposite holds, i.e., $v_i$ leads to a left turn 
in the path $\alpha$, then we connect the parent of $v_s$ with the left child 
of 
$v_s$. 

\end{enumerate}

Intuitively, when traversing over the second part of the path $\alpha_s$ we can 
potentially come across two scenarios. In the first, we reach a feature 
within the path that did not participate in $\alpha_t$. In this case, we want 
to continue traversing both possible scenarios (leaving the corresponding 
feature in the tree). In the second case, we reach a feature that already 
participated in $\alpha_t$. In this scenario, we delete the corresponding node 
from $\alpha_s$ and connect it in the same direction in which the corresponding 
feature is connected in $\alpha_s$.

Hence, at the end of this recursive process, we are left with an equivalent 
diagram where for each path $\alpha$, no two nodes $v_i,v_j\in 
\alpha$ exist, such that $x_v(v_i)=x_v(v_j)$. 
Figure~\ref{fig:combineFbddsThreeSteps} 
depicts an illustration of this recursive process.

\begin{figure}[ht]
\centering
\includegraphics[width=0.48\textwidth]{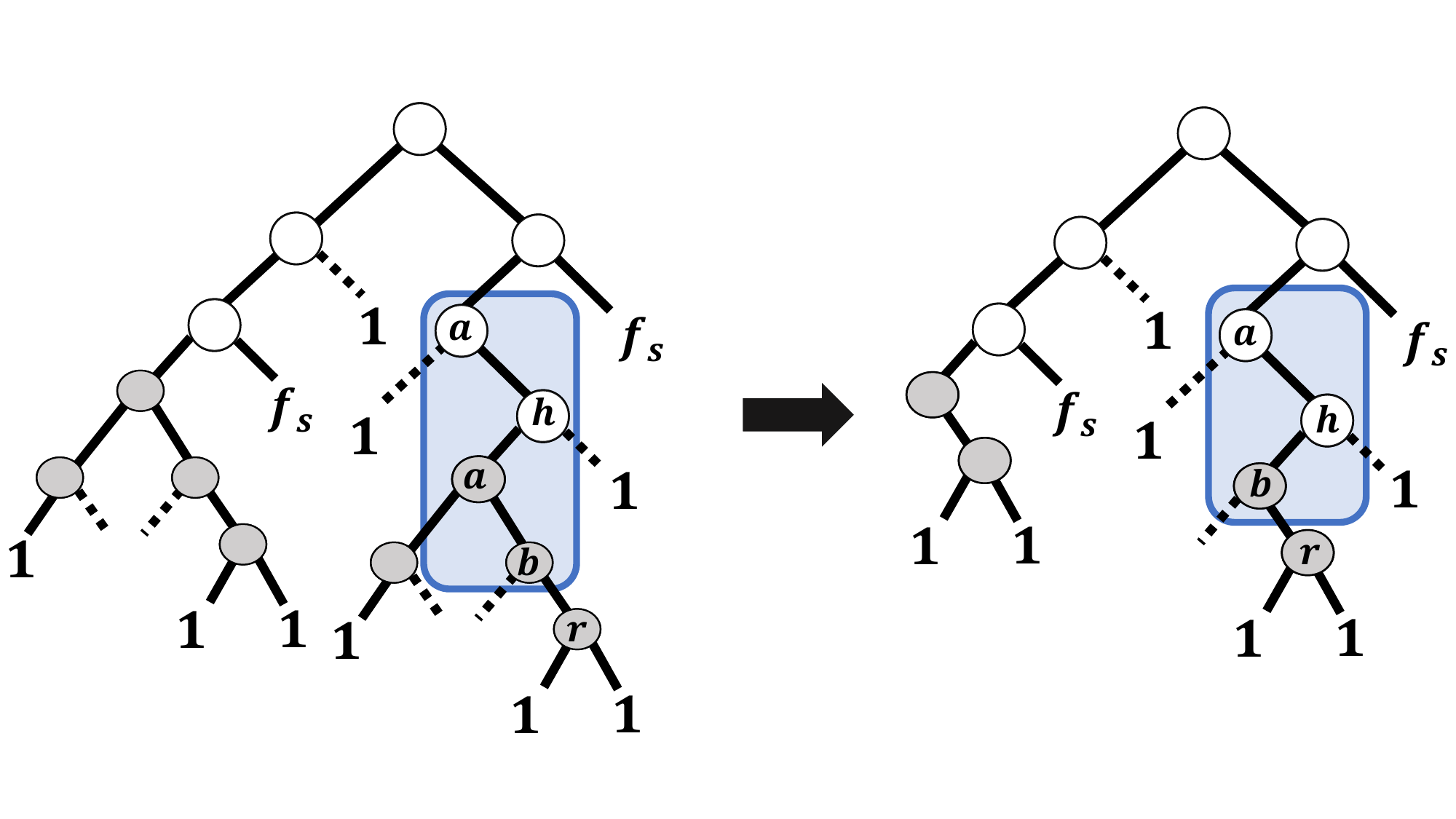}
\caption{An illustration of the polynomial construction of $f_k\in 
\mathcal{C}_{\fbdd}$, given $f'$. The blue boxes represent an area pruned 
during 
our recursive procedure, in order to construct a valid FBDD, without a 
repetition of features.}
\vspace{20pt}
\label{fig:combineFbddsThreeSteps}
\end{figure}
This concludes the proof that given two functions $f_t, f_s\in 
\mathcal{C}_{\fbdd}$ we can polynomially construct some $f_k\in 
\mathcal{C}_{\fbdd}$, which is sufficient to show that $\mathcal{C}_{\fbdd}$ is 
self-aligned.

\begin{lemma}
The class $\mathcal{C}_{\mlp}$ is self-aligned.
\end{lemma}
\medskip\noindent
\emph{Proof.}
We rely on Proposition~\ref{appendix:expressivnes} and show the corresponding 
encodings of $\wedge$ and $\rightarrow$ for $\mathcal{C}_{\mlp}$ (a sufficient 
condition for the self-alignment of the class). 
We start by showing how, given two functions $f_t,f_s\in \mathcal{C}_{\mlp}$ 
over $\{0,1\}^n$, we can polynomially construct a new function $f_k\in 
\mathcal{C}_{\mlp}$, such that $f_k = f_t \lor f_s$. This will later imply the 
existence of the aforementioned logic relations. More formally:

\begin{lemma}
Let $f_{t},f_{s}\in \mathcal{C}_{\mlp}$, then $f_k = f_t \lor f_s\in 
\mathcal{C}_{\mlp}$ can be constructed in polynomial time.
\end{lemma}
\medskip\noindent
\emph{Proof.}
We assume that $f_{t}$ consists of $t$ hidden layers, while $f_{s}$ consists of 
$s$ hidden layers.
We recall that our definition of an MLP is defined recursively, or in other 
words, $f_{t}\coloneqq h_{1}^{(t)}$ is defined as:

\begin{equation}
h_{1}^{(j)} \coloneqq \sigma_{1}^{(j)}(h_{1}^{(j-1)}W_{1}^{(j)} + b_{1}^{(j)}) 
\quad (j \in \{1,\ldots,t\})
\end{equation}

and $f_{s}\coloneqq h_{2}^{(s)}$ is defined as:

\begin{equation}
h_{2}^{(j)} \coloneqq \sigma_{2}^{(j)}(h_{2}^{(j-1)}W_{2}^{(j)} + b_{2}^{(j)}) 
\quad (j \in \{1,\ldots,s\})
\end{equation}

We now construct $f_{k}\coloneqq h_{3}^{(k)}$ where $k\coloneqq max(s,t)+1$. 
For fully 
formulating $h_3$ we need to formulate the associated weights $W_3^{(i)}$ and 
bias terms $b_3^{(i)}$ for each layer $i$. Notice that $i=0$ is the input layer 
so no corresponding bias or weight is defined for them, and the dimensions 
of the input layers of both $h_1$ and $h_2$ are equal (for MLPs that receive 
inputs from the same domain).

Assuming that $s=t$, we construct $b_{3}^{(i)}\coloneqq b_{1}^{(i)}\cdot 
b_{2}^{(i)}$ for all $1\leq i\leq s$. In other words, $b_3$ is a concatenation 
of $b_1$ and $b_2$. Notice that given that the dimensions of the \emph{hidden} 
layers corresponding to $h_1$ are $d_1^1,\ldots, d_1^s$ and the corresponding 
dimensions of the hidden layers corresponding to $h_{2}$ are $d_2^1,\ldots, 
d_2^s$, then the corresponding dimensions of the hidden layers of $h_3$ in this 
case are: $(d^1_1+d^1_2),\ldots, (d_1^s+d_2^s)$.

Now, assuming that $s\neq t$, without loss of generality, we can assume that 
$t>s$. For any layer 
$1\leq i \leq s$ we construct $b_3$ in the same way, i.e., 
$b_{3}^{(i)}\coloneqq b_{1}^{(i)}\cdot b_{2}^{(i)}$. For any $s < i \leq t$ we 
construct: $b_{3}^{(i)}\coloneqq b_{1}^{(i)}\cdot (0)$. In other words, we 
concatenate the vector $b_1$ with a single bias term $0$. In this particular 
case, the corresponding dimensions of $h_3$ are: 
$(d^1_1+d^1_2),\ldots, (d_1^s+d_2^s), (d_1^{s+1}+1), \ldots, (d_1^{t}+1)$.

Finally, we construct the weight vector $W_3^{(i)}\in 
\mathbb{Q}^{d^{i-1}_{3}\times d^{i}_{3}}$. Again, for the case where $s=t$, we 
construct it such that for any $1\leq i \leq s$:

\begin{equation}
\label{matrix_formulation}
W_3^{(i)} = 
\begin{pmatrix} W_1^{(i)} & 0\\ 0 & W_2^{(i)} \end{pmatrix}
\end{equation}

Assuming, without loss of generality, that $s<t$, we construct it such that for 
any $1\leq i \leq 
s$ then $W_3^{(i)}$ is formalized as in Equation~\ref{matrix_formulation} and 
for any $s < i \leq t$ then: 

\begin{equation}
W_3^{(i)} = 
\begin{pmatrix} W_1^{(i)} & 0\\ 0 & 1 \end{pmatrix}
\end{equation}

Intuitively, the construction of the bias terms and the weight matrix $W_3^{i}$ 
layers captures a situation where we ``stack'' the hidden layers of $h_1$ and 
$h_2$. This is a result of the concatenated bias vectors at each step, as well 
the fact that we zero out the effect of the weights corresponding to $h_1$ with 
those of $h_2$, and vice versa. 

Now, we describe the construction for the last layer $k$ of $h_3$ (recall 
that $k\coloneqq max(s,t)+1$). Since we focus on binary classification, the 
single 
activation function $\sigma$ of the last layer can be considered, without loss 
of generality, as 
a step function. We also define the bias to be zero, i.e., $b_3^{k+1}=0$. Now, 
we are left with defining the weight matrix of the last layer. Formally, we 
define $W_3^{(k)}\in \mathbb{Q}^{d^{k-1}_{3}\times d^{k-1}_{3}}$, where each 
weight in $W_3^{(k)}$ is some \emph{strictly} positive weight $w^+_3$.

We now prove that the above encoding of $f_k$ satisfies that $f_k = f_t \lor 
f_s$ for any value $\z\in \mathbb{F}$:
\begin{align*}
f_k(\z) = h_{3}^{(k)}(\z) = 0 
\underset{(*)}{\iff}  \\ 
step_{3}(w_3^{+}\cdot\relu(h_{1}^{(t)}) + w_3^{+}\cdot\relu(h_{2}^{(s)})) (\z) 
= 0 
\underset{(**)}{\iff}  \\ 
(w_3^{+}\cdot\relu(h_{1}^{(t)}) + w_3^{+}\cdot\relu(h_{2}^{(s)})) (\z) \leq 0 
\iff  \\ 
w_3^{+}\cdot(\relu(h_{1}^{(t)}) + \relu(h_{2}^{(s)})) (\z) \leq 0 
\underset{(***)}{\iff}  \\ 
(\relu(h_{1}^{(t)}) + \relu(h_{2}^{(s)})) (\z) \leq 0 
\iff  \\ 
(\relu(h_{1}^{(t)}) + \relu(h_{2}^{(s)})) (\z) = 0 
\underset{(****)}{\iff}  \\ 
\relu(h_{1}^{(t)}(\z))=0 \wedge \relu(h_{2}^{(s)}(\z))=0
\iff \\
h_{1}^{(t)}(\z) \leq 0 \wedge h_{2}^{(s)}(\z) \leq 0
\iff \\
step_{1}(h_{1}^{t-1}(\z)) = 0 \wedge step_{2}(h_{2}^{s-1}(\z)) = 0 
\iff \\
f_s(\z)=0 \wedge f_k(\z)=0
\end{align*}

Where (*) holds, since $f_k$ is directly connected to the outputs of both $f_s$ 
and $f_t$ (in which their step function was replaced by \relu{} activations). 
This is the case both for when $s=t$ or, without loss of generality, when 
$s<t$, replacing each 
corresponding weight in $W_1^{i}$ with a single neuron.
Equivalence (**) holds directly from the definition of the step function (see 
Equation~\ref{eq:step-definition}), while 
equivalence (***) holds from the fact that $w^+_3 > 0$. Equivalence (****) 
follows from the fact that ReLU is a non-negative function.

We provide a visual illustration of this construction in 
Fig.~\ref{fig:combineMlps}. Note that the figure does not explicitly state the 
inner 
connections between $f_t$ and $f_s$ that are with zero weights.
\begin{figure}[h]
\centering
\includegraphics[width=0.41\textwidth]{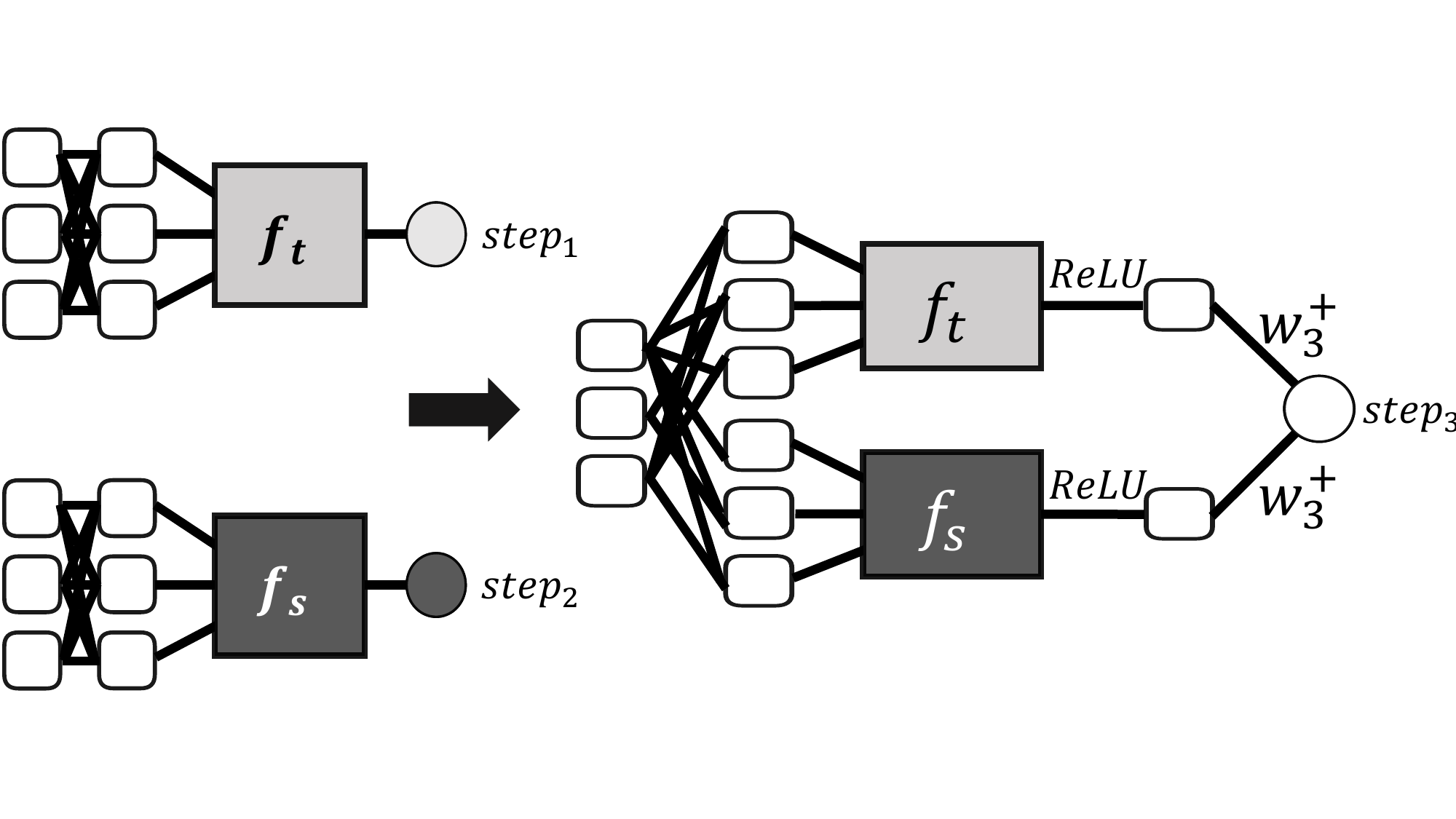}
\caption{An illustration of the effective construction $f_t \lor f_s\in 
\mathcal{C}_{\mlp}$, relying on $f_t, f_s\in \mathcal{C}_{\mlp}$. 
}
\label{fig:combineMlps}
\end{figure}
\bigskip
Now, based on Lemma~\ref{MLPs_symetrically_constructible}, we know that MLPs 
are symmetrically constructible. In other words, given $f_t\in 
\mathcal{C}_{\mlp}$ we can polynomially construct $f_{k}:= \neg f_{t}\in 
\mathcal{C}_{\mlp}$.

Since any logic gate can be encoded using the universal NOR gate, we can now 
polynomially construct both $f_k:= f_s\rightarrow f_t\in \mathcal{C}_{\mlp}$ 
and $f_k:= 
f_s\wedge f_t\in \mathcal{C}_{\mlp}$. This can be done by recursively building 
the 
corresponding MLPs representing either the $f_s\lor f_t$ or $\neg f_s$ 
constructions in a polynomial number of steps. Each one of these steps runs 
in 
polynomial time and outputs a new MLP of size $O(f_s+f_t)$ at each step. Hence, 
we conclude that MLPs are self-aligned.



\medskip
\noindent
Next, we present the formalization of the Subset Sum problem, used for 
proving Proposition~\ref{proposition:MCR_Perceptron_NPcomplete}.
\vspace{0.5em} 

\noindent\fbox{%
	\parbox{\columnwidth}{%
		\mysubsection{SSP (Subset Sum Problem)}:
		
		\textbf{Input}: $(z_1,z_2,\ldots,z_n)$ set of positive integers, 
		integer 
		$k$ (such that $k\leq n$) and a (target) integer $T$.
		
		\textbf{Output}: \yes{}, if there exists a subset $S\subseteq 
		(1,2,\ldots,n)$ of size $k$ such that $\sum_{i\in S}z_i=T$, and \no{} 
		otherwise.
	}%
}

\vspace{0.5em} 

\begin{proposition}
\label{proposition:MCR_Perceptron_NPcomplete}
While the query MCR$(\mathcal{C}_{\perceptron})$ can be solved in polynomial 
time, the query MCR$(\mathcal{C}_{\perceptron},\mathcal{C}_{\perceptron})$ is 
NP-complete.
\end{proposition}

\medskip\noindent
\emph{Proof.}
First, we briefly explain how it is possible to check whether a subset of 
features is contrastive for a Perceptron model, within the misaligned 
configuration, in polynomial time~\cite{BaMoPeSu20}.

A Perceptron is defined by $f=\langle \mathbf{w},b\rangle$, where 
$\mathbf{w}$ is the weight vector corresponding to the input $\x$, and $b$ is 
the bias term. We can obtain the exact value of $\sum_{i\in 
	\Bar{S}}\x_i\cdot w_i$. Then, for the features in $S$, it is possible to 
linearly find the $y$ assignments corresponding to the maximal and minimal 
values of $\sum_{i\in S}y_i\cdot w_i$. The maximal value is obtained by setting 
$y_i:=1$ when $\mathbf{w}_i\geq 0$ and $y_i:=0$ when $\mathbf{w}_i=0$. 
The minimal value is obtained respectively (setting $y_i:=1$ when 
$\mathbf{w}_i< 0$ and $y_i:=0$ when $\mathbf{w}_i\geq0$). Now, we can calculate 
the entire range of possible values that can be obtained 
by setting the values of $\Bar{S}$ to $\x$. If the minimal possible value is 
negative and the maximal possible value is positive then it means that $S$ is 
indeed contrastive, as there exists a subset of features that can alter the 
classification. If not, i.e., the entire range is either strictly positive or 
negative, this means that $S$ is not contrastive. More formally, $S$ is 
contrastive if and only if:

\begin{equation}
	\begin{aligned}
		\label{perceptrons_csr_local}
		\sum_{i\in S}\x_i\cdot w_i+max\{\sum_{i\in \Bar{S}}y_i\cdot w_i+b\ | \ 
		y\in 
		\mathbb{F}\}>0 \ \ \wedge \\
		\sum_{i\in S}\x_i\cdot w_i+min\{\sum_{i\in \Bar{S}}y_i\cdot w_i+b\ | \ 
		y\in 
		\mathbb{F}\}\leq 0 
	\end{aligned}
\end{equation}

In other words, it holds that:
\begin{equation}
	\begin{aligned}
		-max\{\sum_{i\in \Bar{S}}y_i\cdot w_i+b\ | \ y\in \mathbb{F}\}<
		\sum_{i\in S}\x_i\cdot w_i\\ \leq -min\{\sum_{i\in \Bar{S}}y_i\cdot 
		w_i+b\ 
		| \ y\in \mathbb{F}\}
	\end{aligned}
\end{equation}

\mysubsection{Membership.} Given the aforementioned description regarding the 
polynomial validation of contrastive reasons for Perceptrons, membership in NP 
is straightforward, since one can guess a subset $S$ and validate whether it 
holds that $|S|<k$, $S$ is contrastive with respect to $f(\x)$, as well as the 
fact 
that $\pi(\x_S;\z_{\overline{S}})=1$. If these hold then we know that $\langle 
f,\pi,\x,k\rangle 
\in$\emph{MCR}$(\mathcal{C}_{\perceptron},\mathcal{C}_{\perceptron})$.

\mysubsection{Hardness}. We reduce 
\emph{MCR}$(\mathcal{C}_{\perceptron},\mathcal{C}_{\perceptron})$ from SSP 
(the Subset Sum problem), a classic 
NP-complete problem, previously formalized.
Given some $\langle (z_1,z_2,\ldots,z_n),k,T\rangle$ the 
reduction first checks the specific case where $k=n$. In this scenario, 
we want to construct a ``dummy'' result. In other words, if 
$\sum_{i=1}^{n}z_i=T$, we can construct a ``dummy'' instance of $\langle 
f_1,\pi:=f_2,\x:=(1,1),k:=2\rangle$. We define $f_1:=\langle w^1,b_1\rangle$ 
where $w^1:=(1,-1)$ and $b_1:=0$. We define $f_2:=\langle w^2,b_2\rangle$ such 
that $w^1:=(1,1)$ and $b_1:=1$. In case that $\sum_{i=1}^{n}z_i\neq T$, we can 
simply construct a false encoding.

If this is not the case (meaning that $k\neq n$), the reduction constructs the 
two following Perceptrons $f_1:=\langle \mathbf{w^1},b_1\rangle$ and 
$f_2:=\langle \mathbf{w^2},b_2\rangle$, where 
$\mathbf{w^1}:=(-z_1,-z_2,\ldots,-z_n)$, $b_1:=T+\frac{1}{4}$, 
$\mathbf{w^2}:=(z_1,z_2,\ldots,z_n)$, and $b_2:=-T$. The reduction constructs 
$\langle f:=f_1,\pi:=f_2,\x:=\mathbf{1}_n,k:=k\rangle$, and $\mathbf{1}_n$ 
denotes a unit vector of size $n$.

First, notice that there exists a contrastive reason of size $k$ for $f_1$ 
aligned by $f_2$ if and only if:

\begin{equation}
	\begin{aligned}
		\exists S\in (1,\ldots n), \z\in \mathbb{F}. \quad  |S|\leq k \ \wedge 
		\ \\ 
		[f_2(\x_{\Bar{S}};\z_{S})=1\wedge f_1(\x_{\Bar{S}};\z_{S})\neq f_1(\x)]
	\end{aligned}
\end{equation}

This means that there exists a subset $S$ of size $k$ such that:

\begin{equation}
	\begin{aligned}
		[-max\{\sum_{i\in S}y_i\cdot \mathbf{w}^1_i+b_1\ | \ y\in \mathbb{F}\}<
		\sum_{i\in \Bar{S}}\x_i\cdot \mathbf{w}^1_i\\ \leq -\min\{\sum_{i\in 
			S}y_i\cdot \mathbf{w}^1_i+b_1\ | \ y\in \mathbb{F}\}] \wedge \\
		[\sum_{i\in \Bar{S}}\x_i\cdot \mathbf{w}^2_i \leq -\min\{\sum_{i\in 
			S}y_i\cdot \mathbf{w}^2_i+b_2\ | \ y\in \mathbb{F}\}] 
	\end{aligned}
\end{equation}

We assume a valid encoding (since this can trivially be validated in polynomial 
time). The first ``dummy'' validation checks whether $k=n$ and 
$\sum_{i=1}^{n}z_i=T$. In such a case, it holds that $\langle 
(z_1,z_2),k,T\rangle\in SSP$. We also note that $S$, which is of size $k=2$ 
(the entire 
input domain) is an aligned contrastive reason. This is due to the fact that 
$f_1((1,1))=1$ and $f_1((0,1))=0$. Additionally, it holds that for any $\z$ it 
satisfies that $f_2(\z)=1$. In other words:

\begin{equation}
	\begin{aligned}
		\exists S\in (1,\ldots n), \z\in \mathbb{F}. \quad  |S|\leq k \ \wedge 
		\ \\ 
		[f_2(\x_{\Bar{S}};\z_{S})=1\wedge f_1(\x_{\Bar{S}};\z_{S})\neq f_1(\x)]
	\end{aligned}
\end{equation}

Hence $\langle f,\pi,\x,k\rangle\in 
\mcr(\mathcal{C}_{\perceptron},\mathcal{C}_{\perceptron})$. 

Thus, we can now assume that $k<n$. Since all values in $\mathbf{w}^1$ are 
negative, then it holds that for any subset of features $S$:
\begin{equation}
	max\{\sum_{i\in S}y_i\cdot \mathbf{w}^1_i+b_1\ | \ y\in \mathbb{F}\} = b_1
\end{equation} 

It also clearly holds that for any $S$:

\begin{equation}
	b_1\geq min\{\sum_{i\in S}y_i\cdot \mathbf{w}^1_i+b_1\ | \ y\in \mathbb{F}\}
\end{equation} 

Since we know that $k<n$, and since these are negative \emph{integers}, then 
there exists at least one integer in the complementary set $\overline{S}$, and 
hence it also holds that:
\begin{equation}
	\begin{aligned}
		max\{\sum_{i\in S}y_i\cdot \mathbf{w}^1_i+b_1\ | \ y\in \mathbb{F}\} = 
		b_1>\\ min\{\sum_{i\in S}y_i\cdot \mathbf{w}^1_i+b_1\ | \ y\in 
		\mathbb{F}\}+\frac{1}{2}
	\end{aligned}
\end{equation}

Regarding $\mathbf{w}^2$, since all of its values are positive, then it holds 
that:

\begin{equation}
	\begin{aligned}
		min \{\sum_{i\in S}y_i\cdot \mathbf{w_i^2}+b_2\ | \ y\in \mathbb{F}\} = 
		b_2
	\end{aligned}
\end{equation}

Now, if $\langle (z_1,z_2,\ldots,z_n), T\rangle\in SSP$, then there exists a 
subset 
of features $S\subseteq (1,2,\ldots,n)$ of size $k$ such that $\sum_{i\in 
	S}z_i=T$. This indicates that:

\begin{equation}
	\begin{aligned}
		\sum_{i\in \Bar{S}}\x_i\cdot \mathbf{w^1_i}=-T=-b_1+\frac{1}{4}>-b_1 = 
		\\ 
		-max\{\sum_{i\in S}y_i\cdot \mathbf{w^1_i}+b_1\ | \ y\in \mathbb{F}\}
	\end{aligned}
\end{equation}

as well as: 

\begin{equation}
	\begin{aligned}
		\sum_{i\in \Bar{S}}\x_i\cdot \mathbf{w^1_i}=-T=-b_1+\frac{1}{4}<-b_1 
		+\frac{1}{2} < \\ -min\{\sum_{i\in S}y_i\cdot \mathbf{w^1_i}+b_1\ | \ 
		y\in 
		\mathbb{F}\}
	\end{aligned}
\end{equation}

Regarding $\mathbf{w}^2$, it holds that:

\begin{equation}
	\begin{aligned}
		\sum_{i\in \Bar{S}}\x_i\cdot \mathbf{w^2_i}=T=-b_2\\= -min\{\sum_{i\in 
			S}y_i\cdot \mathbf{w^2_i}+b_2\ | \ y\in \mathbb{F}\}
	\end{aligned}
\end{equation}

and hence: 
\begin{equation}
	\langle f,\pi,\x,k\rangle \in 
	\mcr(\mathcal{C}_{\perceptron},\mathcal{C}_{\perceptron})
\end{equation}

Given that $\langle (z_1,z_2,\ldots,z_n), T\rangle\not\in SSP$, then there 
does not exist a subset of features $S\subseteq (1,2,\ldots,n)$ of size $k$ 
such that $\sum_{i\in S}z_i=T$. This is equivalent to saying that there does 
not exist a subset of features $S\subseteq (1,2,\ldots,n)$ of size $k$ such 
that:

\begin{equation}
	\begin{aligned}
		(\sum_{i\in \Bar{S}}\x_i\cdot \mathbf{w^1_i}\leq -T)\wedge 
		(\sum_{i\in\Bar{S}}\x_i\cdot \mathbf{w^2_i}\geq T)
	\end{aligned}
\end{equation}

meaning that there does not exist a subset $S$ such that:

\begin{equation}
	\begin{aligned}
		[-max\{\sum_{i\in S}y_i\cdot \mathbf{w}^1_i+b\ | \ y\in \mathbb{F}\}<
		\sum_{i\in \Bar{S}}\x_i\cdot \mathbf{w}^1_i] \wedge \\
		[\sum_{i\in \Bar{S}}\x_i\cdot \mathbf{w}^2_i \leq -\min\{\sum_{i\in 
			S}y_i\cdot \mathbf{w}^2_i+b\ | \ y\in \mathbb{F}\}] 
	\end{aligned}
\end{equation}

and hence: $\langle f,\pi,\x,k\rangle \not\in 
\mcr(\mathcal{C}_{\perceptron},\mathcal{C}_{\perceptron})$, concluding the 
reduction.

\begin{theorem}
\label{appendix:perceptons_not_aligned}
Assuming that $P\neq NP$, the class $\mathcal{C}_{\perceptron}$ is not 
self-aligned.
\end{theorem}

\medskip\noindent
\emph{Proof.}
Assume, by negation, that the model class $\mathcal{C}_{\perceptron}$ is 
self-aligned. 
Hence, we deduce that given some $f,\pi\in \mathcal{C}_{\perceptron}$ we can 
polynomially construct a function $g$ such that $\langle 
f,\pi,\x,I\rangle\in\mathbf{Q}(\mathcal{C}_{\perceptron},\mathcal{C}_{\perceptron})$
 if and only if $\langle g,\x,I\rangle\in 
\mathbf{Q}(\mathcal{C}_{\perceptron})$. As we 
know from Proposition~\ref{proposition:MCR_Perceptron_NPcomplete} that deciding 
$\mathbf{Q}(\mathcal{C}_{\perceptron},\mathcal{C}_{\perceptron})$ is 
NP-hard and that  $\mathbf{Q}(\mathcal{C}_{\perceptron})$ is in PTIME, the 
following claim holds only if PTIME=NP. Hence, $\mathcal{C}_{\perceptron}$ is 
not self-aligned.


\section{Framework Extenstions}
\label{appendix:extensions}

Our framework can be extended in multiple axes. First, although we follow 
common conventions 
(e.g.,~\cite{arenas2022computing,BaMoPeSu20,waldchen2021computational, BaAmKa24}) and 
focus on \emph{binary} input and output domains to simplify our presentation, 
some of our findings 
are applicable to \emph{any} discrete input or output domains. Additionally, 
rather 
than 
considering 
individual features, we can consider ``high-level" features, e.g., by grouping 
multiple features together. This allows defining explanations, for example, in 
terms of 
super-pixels and also RGB images in various practical settings.

More broadly, the types of explanations analyzed in this work are explanations 
with formal and mathematical guarantees, commonly discussed within a sub-field 
of interest known as \emph{Formal XAI}~\cite{marques2022delivering}. One 
benefit of explanations with formal guarantees is that, unlike heuristic-based 
explanations, they enable a more rigorous and mathematical analysis, allowing 
the study of computational complexity aspects of obtaining 
explanations~\cite{BaMoPeSu20, BaKa23}.

However, there exists a body of work in Formal XAI that focuses on the 
practical aspect of computing explanations with formal 
guarantees~\cite{darwiche2020reasons, marques2020explaining, BaKa23, 
BaAmCoReKa23, ignatiev2019abduction, darwiche2023complete, 
darwiche2022computation, huang2021efficiently, izza2021explaining}. Initial 
efforts to compute such explanations were demonstrated on simple ML models, 
which allow tractable computations of explanations. These models include 
decision trees~\cite{IzIgMa22, huang2021efficiently}, linear 
models~\cite{marques2020explaining}, monotonic 
classifiers~\cite{marques2021explanations}, and tree 
ensembles~\cite{IgIzStMa22, izza2021explaining, audemard2023computing, 
boumazouza2021asteryx}. More recently, various methods have been proposed to 
obtain explanations with formal guarantees for neural 
networks~\cite{LaZbMiPaKw21, WuWuBa24, BaAmCoReKa23, BaKa23, HuMa23, 
izza2024distance}. This task is considered a computationally challenging 
one~\cite{BaMoPeSu20, BaAmKa24}. However, the development of such methods is 
facilitated by the rapid advancement of neural network verification 
tools~\cite{KaBaDiJuKo17, KaHuIbJuLaLiShThWuZeDiKoBa19, 
WuIsZeTaDaKoReAmJuBaHuLaWuZhKoKaBa24, SuKhSh19, LyKoKoWoLiDa20, GeMiDrTsChVe18, 
ZhShGuGuLeNa20, JaBaKa20, KoLoJaBl20, AmScKa21, CoYeAmFaHaKa22, YeAmElHaKaMa22, 
YeAmElHaKaMa23, AmCoYeMaHaFaKa23}, which are being developed more broadly to 
formally certify a diverse set of provable properties~\cite{CoAmKaFa24, 
CoAmRoSaKaFo24, MaAmWuDaNeRaMeDuGaShKaBa24, MaAmWuDaNeRaMeDuHoGaShKaBa24, 
RoAmCoSaKa24, Eh17, BuTuToKoMu18, TrBkJo20, ElGoKa20, CaKoDaKoKaAmRe22, 
AmWuBaKa21, AmZeKaSc22, AmFrKaMaRe23, AmMaZeKaSc23, MaCoCiFa23}. We believe 
that our work can serve as a step towards developing a more 
rigorous understanding of the potential capabilities and limitations of 
computing explanations with formal guarantees concerning a given distribution 
of interest.

\end{document}